%% file: main.tex
\documentclass[10pt,twocolumn,letterpaper]{article}

\usepackage{cvpr}              
\usepackage[accsupp]{axessibility} 


\definecolor{cvprblue}{rgb}{0.21,0.49,0.74}
\usepackage[pagebackref,breaklinks,colorlinks,allcolors=cvprblue]{hyperref}
\input{packages}
\def\paperID{3643} 
\def\confName{CVPR}
\def\confYear{2025}

\title{Project-Probe-Aggregate: Efficient Fine-Tuning for Group Robustness}

\author{
 {Beier Zhu}$^{1,2}$ \quad 
 {Jiequan Cui}$^1$ \quad
 {Hanwang Zhang}$^1$ \quad
 {Chi Zhang}$^{2\dagger}$\\
$^1$Nanyang Technological University, $^2$Westlake University \\
\tt \small 
\{beier.zhu, jiequan.cui, hanwangzhang\}@ntu.edu.sg, chizhang@westlake.edu.cn}

\begin{document}

\input{define}

\maketitle
\input{sec/0_abstract}    
\input{sec/1_intro}
\input{sec/2_related}
\input{sec/3_method}
\input{sec/4_justify}
\input{sec/5_experiments}
\input{sec/6_conclusion}

{
    \small
    \bibliographystyle{ieeenat_fullname}
    \bibliography{main}
}

\newpage
\mbox{}
\newpage
\appendix

\tableofcontents 
\input{appendix/1_proofs}
\input{appendix/3_additional_experimental_details}
\input{appendix/4_additional_datasets}
\end{document}

%% file: packages.tex
\usepackage[export]{adjustbox}
\usepackage{multirow}
\usepackage{pifont}
\usepackage{bm}
\usepackage{dsfont}
\usepackage{amsmath}
\usepackage{amssymb}
\usepackage{amsthm}
\usepackage{algorithm}
\usepackage{algpseudocode}
\usepackage{color, colortbl}

\DeclareMathOperator*{\argmax}{argmax}

%% file: define.tex
\newcommand{\ours}{\text{PPA}}
\newcommand{\x}{\mathbf{x}}
\newcommand{\z}{\mathbf{z}}
\newcommand{\vc}{\mathbf{c}} 
\newcommand{\s}{\mathbf{s}} 

\newtheorem{definition}{Definition}
\newtheorem{theorem}{Theorem}
\newtheorem{assumption}{Assumption}
\newtheorem{lemma}{Lemma}
\newtheorem{proposition}{Proposition}
\newtheorem{corollary}{Corollary}

\newcommand{\tableCellHeight}{1}
\newcommand{\tabstyle}[1]{
  \setlength{\tabcolsep}{#1}
  \renewcommand{\arraystretch}{\tableCellHeight}
  \centering
  \small
}

\definecolor{tabhighlight}{HTML}{e5e5e5}

\newtheoremstyle{restatedlemma}
  {\topsep}       
  {\topsep}       
  {\itshape}      
  {}              
  {\bfseries}     
  {.}             
  {.5em}          
  {\thmname{#1} \thmnumber{#2} (\thmnote{#3})} 

\newtheoremstyle{restatedproposition}
  {\topsep}       
  {\topsep}       
  {\itshape}      
  {}              
  {\bfseries}     
  {.}             
  {.5em}          
  {\thmname{#1} \thmnumber{#2} (\thmnote{#3})} 

\theoremstyle{restatedlemma}
\newtheorem*{restatedlemma}{Restated Lemma}

\theoremstyle{restatedproposition}
\newtheorem*{restatedproposition}{Restated Proposition}

%% file: sec/0_abstract.tex
\renewcommand{\thefootnote}{}
\footnotetext{This work was done during Beier Zhu's visit  at AGI Lab, Westlake University. $^{\dagger}$ denotes corresponding author.}
\renewcommand{\thefootnote}{\arabic{footnote}}

\begin{abstract}
While image-text foundation models have succeeded across diverse downstream tasks, they still face challenges in the presence of spurious correlations between the input and label. 
To address this issue, we propose a simple three-step approach--Project-Probe-Aggregate ($\ours$)--that enables parameter-efficient fine-tuning for foundation models without relying on group annotations. 
Building upon the failure-based debiasing scheme, our method, $\ours$, improves its two key components: minority samples identification and the robust training algorithm.
Specifically, we first train biased classifiers by projecting image features onto the nullspace of class proxies from text encoders. Next, we infer group labels using the biased classifier and probe group targets with prior correction. Finally, we aggregate group weights of each class to produce the debiased classifier. Our theoretical analysis shows that our PPA enhances minority group identification and is Bayes optimal for minimizing the balanced group error, mitigating spurious correlations. Extensive experimental results confirm the effectiveness of our $\ours$: it outperforms the state-of-the-art by an average worst-group accuracy while requiring less than $0.01\%$ tunable parameters without training group labels.  
\end{abstract}

%% file: sec/1_intro.tex
\section{Introduction}
\label{sec:intro}
Image-text foundation models~\cite{radford2021learning,jia2021scaling,wang2022simvlm}—large models pre-trained on web-scale data—are becoming increasingly prevalent in real-world deployments. However, adapting these models to downstream tasks, whether through zero-shot transfer or fine-tuning, remains challenging due to fundamental issues of group robustness: they achieve low average test error but incur high risk on certain groups of examples~\cite{chuang2023debiasing,zhang2022contrastive,you2024calibrating,yang2023mitigating}. This issue becomes particularly pronounced in the presence of spurious correlations~\cite{geirhos2020shortcut}, where a classifier relies on non-causal relationships that happen to align with class labels in training but break down during deployment. For example, consider the case where “waterbirds” predominantly appear in images with a “water” background. When trained on such a dataset, neural networks often rely on the background to recognize the objects.

As a result, several fine-tuning methods have emerged to improve group robustness in foundation models~\cite{chuang2023debiasing,zhang2022contrastive,you2024calibrating,yang2023mitigating,phancontrollable}. Unlike conventional robust training frameworks, which typically re-train the entire model~\cite{liu2021just,liuavoiding,nam2020learning,namspread} (which is computationally expensive) and/or require group labels~\cite{sagawa2019distributionally,arjovsky2019invariant} (which demands additional annotation costs), these fine-tuning methods prioritize parameter and annotation efficiency. They address spurious correlations by optimizing only a small subset of parameters without relying on training group labels. For instance, Zhang and Ré~\cite{zhang2022contrastive} fine-tune a lightweight adapter and You \etal \cite{you2024calibrating} propose to optimize projection layers to improve distributional robustness. In this paper, we build on these ideas by focusing on efficient fine-tuning to enhance group robustness, \ie, training linear probes.

\input{tables/intro_table}
At a high level, in the absence of group annotations, a common strategy~\cite{li2019repair,nam2020learning,liuavoiding,liu2021just,zhang2022contrastive,zhang2022correct,dagaev2023too} involves first identifying minority groups by training a \textit{biased model} that easily overfits to spurious features. Subsequently, a \textit{debiased model} is trained using the inferred group labels. To improve the identification of minority groups, prior work proposes various strategies: \cite{nam2020learning,liuavoiding} employ generalized cross-entropy loss to amplify the model bias, JTT~\cite{liu2021just} simply trains a model with empirical risk minimization (ERM) and tunes hyperparameters to exacerbate bias, CA~\cite{zhang2022contrastive} leverages the correctness of zero-shot predictions and \cite{dagaev2023too} uses low-capacity networks to capture spurious features. In this paper, we leverage the pre-trained knowledge from the image-text foundation models to effectively induce bias, further enhancing the identification of minority groups. Specifically, we \textbf{project} out class proxies, provided by the text encoders, from the input image features and use only the remaining information to train the biased classifiers. Intuitively, removing class information forces the model to rely more heavily on spurious features for predictions. We provide theoretical analysis to support this intuition in~\cref{prop:weight}. Empirically, in~\cref{tab:recall-precision}, we evaluate the quality of the inferred groups in the training datasets obtained by our method, and compare it to prior approaches~\cite{zhang2022contrastive,liu2021just}. Specifically, we measure precision—the fraction of examples within the identified minor groups that belong to the worst-performing group, and recall—the fraction of examples from the worst-performing group that are captured by the identified minor groups. As expected, our method achieves higher precision and recall, demonstrating superior group identification.

Given the inferred groups, typical debiased learning algorithms up-weight/up-sample minority group samples to enhance performance on these challenging instances~\cite{liu2021just,kirichenko2022last,idrissi2022simple,nam2020learning,dagaev2023too}. However, determining the optimal weights requires careful hyper-parameter tuning and lacks theoretical guidance. Since the goal of debiased training is to minimize the Balanced Group Error (BGE)~\cite{liuavoiding,tsirigotis2024group}, we propose a methodology designed to achieve this objective directly. Rather than predicting class labels, our method is trained to \textbf{probe} for group targets by compensating group priors into the softmax cross-entropy during training. Afterward, we \textbf{aggregate} group weights of each class to produce the final debiased classifier.  
We dub our overall debiasing framework Project-Probe-Aggregate ({PPA}), and demonstrate that it is \textit{Bayes} optimal for minimizing BGE in~\cref{propo:gla}. 

In summary, our contributions are:
\begin{itemize}
    \item We project out the class proxies offered by the text encoder and train the biased model on the residual image features, with theoretical analysis showing this approach increases reliance on spurious features, improving minority group identification.
    \item We propose a novel debiasing method that is consistent for minimizing the balanced group error. The debiased classifier predicts group labels compensated by group priors and uses weight-space aggregation after fine-tuning to avoid group inference overhead.
    \item We confirm the effectiveness of our proposal across five benchmarks with spurious correlations. With less than $0.01\%$ of the foundation model parameters, our approach $\ours$ improves the worst group accuracy over the state-of-the-art without training group annotations. Beyond linear probes, we also adapt other parameter-efficient fine-tuning paradigms (e.g., prompt tuning and adapters) to show the versatility of the proposed method.
\end{itemize}

%% file: tables/intro_table.tex
\begin{table}[t]
\tabstyle{5.5pt}
\centering
\begin{tabular}{ccc}
\toprule
\begin{tabular}[c]{@{}c@{}}Method \end{tabular} &
  \begin{tabular}[c]{@{}c@{}}Worst-group\\Recall (\%)\end{tabular} &
  \begin{tabular}[c]{@{}c@{}}Worst-group\\Precision (\%)\end{tabular}  \\
  \midrule
CA~\cite{zhang2022contrastive}  & 46.4 & 15.6  \\
JTT~\cite{liu2021just}     & 78.6  & 24.1 \\
$\ours$ (ours)  & 80.4 & 44.3 \\
\bottomrule
\end{tabular}
\caption{
        The precision and recall of the worst-group examples (\ie, the group with lowest validation accuracy) identified by biased models on the Waterbird training set~\cite{sagawa2019distributionally} using CLIP ResNet-50. Our $\ours$ is more susceptible to suprious features, achieving higher worst-group recall and precision compared to previous methods.      
}
\label{tab:recall-precision}
\end{table} 

%% file: sec/2_related.tex
\section{Related Work}
\label{sec:related_work}

\noindent\textbf{Improving group robustness.}
There have been numerous works aimed at improving group robustness, which broadly fall into two categories based on whether group labels are available during training.
(i) \textit{Group supervised methods} focus on balancing the contribution of each group during the training phase, such as through re-sampling/re-weighting~\cite{idrissi2022simple,kirichenko2022last,izmailov2022feature,menon2020overparameterisation,jung2022learning,wang2020towards}, or robust optimization~\cite{arjovsky2019invariant,sagawa2019distributionally,taghanaki2021robust}. However, this approach requires additional annotation costs for group labels. In this paper, we focus on (ii) \textit{group unsupervised methods}, which do not assume the availability of group labels during training~\cite{asgari2022masktune,kim2021biaswap,kim2022learning,yaghoobzadeh2019increasing,lee2021learning,wu2023discover,yang2022understanding,tartaglione2021end}. 
Without group labels, a common strategy is to first train an easily biased model, and then infer groups based on its biased predictions. A robust model is subsequently trained using the inferred group labels through techniques such as importance weighting~\cite{liu2021just,nam2020learning,qiu2023simple}, robust optimization~\cite{creager2021environment,namspread}, or by learning similar representations for groups within the same class~\cite{zhang2022correct}.  In this paper, we leverage the general knowledge from image-text foundation models to train a biased classifier by excluding the class proxy information generated by the CLIP text encoder. Note that
  classifier in~\cite{wang2020towards} also predicts group labels, our approach differs in \textit{two key aspects}: 1) We further incorporate logit offsets during training, which compensates for group imbalances and leads to lower balanced error. 2) Unlike~\cite{wang2020towards}, our method \textit{does not} require access to ground-truth group labels.
  
\noindent\textbf{Efficient robust fine-tuning for foundation models.} 
The advent of image-text foundation models has led to several efforts to develop efficient group-robust fine-tuning methods~\cite{kirichenko2022last,zhang2022contrastive,chuang2023debiasing,yang2023mitigating,you2024calibrating,phan2024controllable,zhang2023diagnosing,espinosa2024efficient,kim2024discovering,zhu2024robust}. Kirichenko \etal~\cite{kirichenko2022last} demonstrate that last-layer retraining significantly improves group robustness in pretrained models. Zhang and R\'e~\cite{zhang2022contrastive} use contrastive objectives to train adapters, bringing same-class embeddings closer together. Chuang \etal~\cite{chuang2023debiasing} reduce bias in image-text models by removing biased directions in text embeddings. Other works~\cite{yang2023mitigating,you2024calibrating} focus on lightweight fine-tuning by optimizing only the projection layers of the vision branch for distributional robustness.
Our approach builds on efficient fine-tuning by training only a linear classification head for group targets and further introduces weight-space ensembling after fine-tuning to eliminate the overhead of group inference.

%% file: sec/3_method.tex
\section{Preliminaries}

\subsection{Setup}
Considering a classification problem with instances $\x \in \mathcal{X} \subseteq \mathbb{R}^d$ and labels $y \in \mathcal{Y}=[K]$. 
Each data point $(\x,y)$ has an input attribute $a(\x) \in \mathcal{A}$ that is spuriously correlated with the target $y$. 
We form $|\mathcal{G}|=|\mathcal{A}|\times |\mathcal{Y}|$ groups, where each group label $g \in \mathcal{G}$ is the combination of the attribute $a$ and the class label $y$, \ie, $g:=(a, y)$.
Given a training dataset $\mathcal{D}=\{\x_i,y_i\}_{i=1}^N$ drawn from some unknown distribution $\mathbb{P}$, and let $\mathbb{P}_g$ be the distribution conditioned on $g$ for any $g \in \mathcal{G}$.
We assume that spurious attribute annotations are \textit{not} available for the samples in the training set.
Our goal is to train a model $f_{\theta}: \mathcal{X} \rightarrow \mathbb{R}^{K}$ that outputs prediction score, such that it achieves low expected error:
\begin{equation}
    \mathcal{R}=\mathbb{E}_{\x,y\sim \mathbb{P}}[y\neq \argmax_{y'\in \mathcal{Y}}f_{\theta}(\x)_{y'}]
\end{equation}
and, more importantly, exhibits \textit{group robust}, \ie, achieving low worst-group error:
\begin{equation}
   \mathcal{R}_\mathsf{wg} = \max_{g\in \mathcal{G}}\mathbb{E}_{\x,y\sim \mathbb{P}_g} [y\neq \argmax_{y'\in \mathcal{Y}}f_{\theta}(\x)_{y'}]
\end{equation}

For example, in the Waterbirds~\cite{sagawa2019distributionally} dataset, the task is to classify birds $y \in \{\text{waterbird}, \text{landbird}\}$, with the background serving as a spurious attribute ($a \in \{\text{water}, \text{land}\}$). A notable pattern is that about $95\%$ of waterbird images have a water background, causing models to rely on the ``water'' background to predict ``waterbird''. This reliance leads to poor performance on minority groups, such as landbirds with water backgrounds $g = (\text{water}, \text{landbird})$.
Unlike domain generalization or out-of-distribution evaluation, all groups are present in the training, validation, and test splits. However, the training groups are imbalanced, which can result in poor group robustness on the test set.

\subsection{Common Strategy for Improving Unsupervised Group Robustness}\label{sec:commonstrategy}
At a high level, typical methods without access to group labels first train a \textit{biased model} and then use its biased predictions to infer the groups. Subsequently, a \textit{robust model} is trained using the inferred group labels.\footnote{It is also known as the failure-based debiasing scheme.~\cite{nam2020learning}}

Taking JTT~\cite{liu2021just} as an example, it identifies the misclassified samples from ERM model $f_\mathsf{erm}$ as the error set~$\mathcal{E}$:
\begin{equation} 
\mathcal{E}=\{(\x_i,y_i)\ \text{s.t.}\ f_\mathsf{erm}(\x_i)\neq y_i\}, 
\end{equation} 
 then trains a robust model by upweighting the samples in $\mathcal{E}$: 
\begin{equation} 
\mathcal{L}_\mathsf{JTT}(\theta,\mathcal{E})=\lambda\sum_{(\x,y)\in \mathcal{E}}\ell(\x,y;\theta)+\sum_{(\x,y)\notin \mathcal{E}}\ell(\x,y;\theta), 
\end{equation} where $\lambda\in \mathbb{R}_+$ is a hyperparameter.
The core idea behind JTT is that ERM models tend to fit groups with easily learnable spurious correlations. The samples in $\mathcal{E}$ are primarily from minority groups where these correlations fail. Upweighting such challenging samples improves the model's performance on the worst-performing groups. 

\noindent\textbf{Remark.} Improving group robustness without access to group labels hinges on identifying minority groups by training a biased model that overfits on spurious features. Prior work amplifies this bias through various strategies: LfF~\cite{nam2020learning} employs generalized cross-entropy (GCE) loss to bias the first-stage model, JTT~\cite{liu2021just} tunes hyperparameters to induce bias, and \cite{dagaev2023too} uses low-capacity networks to capture spurious features.
In this paper, we enhance bias by projecting out the class proxies from the input features, using only the remaining information for classification.
In~\cref{sec:train_biased}, we prove that this operation increases susceptibility to spurious features, aiding the identification of minority groups.

Another key factor is the robust learning algorithm given the pseudo groups. Prior methods address this through importance weighting~\cite{liu2021just,nam2020learning}, robust optimization~\cite{creager2021environment,namspread}, or aligning representations within the same class~\cite{zhang2022correct}. Our approach, $\ours$, predicts pseudo group labels and applies a group prior offset to correct for imbalance. In~\cref{sec:gla}, we prove that our training algorithm effectively mitigates spurious correlations.

\section{Methods}
Tuning entire pre-trained models is both time-consuming and computationally costly. Recent work~\cite{kirichenko2022last} shows that simple last-layer retraining performs well on spurious correlation benchmarks. In our approach, the input 
$\x$ represents image features encoded by CLIP. We freeze the backbone and learn linear classifiers for efficient fine-tuning. 

We now present $\ours$, a simple three-step approach following the common strategy in~\cref{sec:commonstrategy}. First, we train a biased classifier by removing the class proxies generated by the CLIP text encoder. Next, we group samples by classification correctness and train a robust model to predict group labels using our proposed group logit adjustment. Finally, we aggregate group weights to produce the final classifier. 

\noindent\textbf{Step 1: Project out class proxies.}
We compute the matrix $Z=[\z_1,\dots,\z_K]^\top \in \mathbb{R}^{K\times d}$  whose rows are the text embeddings of the $K$ class names. Specifically, each $\z_j$ is derived from a prompt like ``a photo of a \texttt{[CLASS]}.'' with the class token \texttt{[CLASS]} replaced by the $j$-th class name.
Let $\Pi\in \mathbb{R}^{d\times d}$ be the projection operator onto the null space of $Z$, given by:
\begin{equation}\label{eq:Pi}
    \Pi = I - Z^\top (Z Z^\top)^{-1} Z 
\end{equation}
With $\Pi$, the biased model $f_\mathsf{b}: \mathcal{X} \rightarrow \mathbb{R}^K$ is formulated as a linear classifier parameterized by a matrix $W_\mathsf{b}\in \mathbb{R}^{K\times d}$:
\begin{equation}\label{eq:biased-classifier}
    f_\mathsf{b}(\x) =  W_\mathsf{b}\Pi \x 
\end{equation}
To account for potential class imbalance\footnote{Datasets like Waterbirds~\cite{sagawa2019distributionally} and CelebA~\cite{liu2015deep} are class imbalanced.}, we apply the logit-adjustment loss~\cite{menonlong,zhu2023generalized} to ensure $f_\mathsf{b}$ focuses on spurious features rather than class priors:
\begin{equation}\label{eq:la}
    \ell_\mathsf{la}(y,f_\mathsf{b}(\x))= - \ln \frac{\exp(f_\mathsf{b}(\x)+\ln \bm{\pi})_y}{\sum_{y'\in \mathcal{Y}}\exp(f_\mathsf{b}(\x)+\ln \bm{\pi})_{y'}},
\end{equation}
where $\bm{\pi}_y$ is the prior of class $y$.
In~\cref{prop:weight}, we show that our biased model is more easily influenced by spurious correlation.

\noindent\textbf{Step 2: Probe with group target.}
We follow the heuristic of JTT~\cite{liu2021just} to identify minority groups that $f_\mathsf{b}$ misclassifies:
\begin{equation}\label{eq:pseudo}
    \hat{a}(\x)= \mathds{1}[y\neq \argmax_{y'\in \mathcal{Y}}f_\mathsf{b}(\x)_{y'}],
\end{equation}
where $\mathds{1}[\cdot]$ is the indicator function.
Then, each training sample is augmented as $(\x,y,\hat{a})$, with the pseudo group defined as the combination of $y$ and $\hat{a}$:
$\hat{g}:=(y,\hat{a})\in \hat{\mathcal{G}}$. 

Our approach $\ours$ uses a group scorer, implemented via linear probing, to predict pseudo group labels. Specifically, the group score $h_\mathsf{d}(\x)=W_\mathsf{d} \x: \mathcal{X} \rightarrow \mathbb{R}^{|\hat{\mathcal{G}}|}$ is parameterized by  $W_\mathsf{d} \in \mathbb{R}^{|\hat{\mathcal{G}}| \times d}$. 
Let $\hat{\bm{\beta}}$ denote the group priors, we propose the group logit adjustment loss to achieve balanced group performance.: 
\begin{equation}\label{eq:gla}
    \ell_\mathsf{gla}(\hat{g},h_\mathsf{d}(\x))= - \ln \frac{\exp(h_\mathsf{d}(\x)+\tau \cdot \ln \hat{\bm{\beta}})_{\hat{g}}}{\sum_{g'\in \mathcal{\hat{\mathcal{G}}}}\exp(h_\mathsf{d}(\x)+\tau \cdot \ln \hat{\bm{\beta}})_{g'}},
\end{equation}
where $\tau \in \mathbb{R}_{+}$ is a hyper-parameter.

\noindent\textbf{Step 3: Aggregate weights.}
The final debiased classifier $f_\mathsf{d}(\x): \mathcal{X} \rightarrow \mathbb{R}^{K}$ is constructed by aggregating the weights corresponding to each class:

\begin{equation}\label{eq:aggregate_w}
    f_\mathsf{d}(\x)_y= \mathbf{w}_y^\top \x,\ \text{where}\ \mathbf{w}_y^\top=\sum_{g\in \hat{\mathcal{G}}(y)}W_{\mathsf{d},g}.
\end{equation}
Here, $W_{\mathsf{d},g}$ represents the weight vector in $W_\mathsf{d}$ associated with group $g$. This weight-space aggregation transforms group classification into class prediction, eliminating the need for group inference overhead. 

In~\cref{propo:gla}, we prove that the debiased model $f_\mathsf{d}(\x)$ is the \textit{Bayes} optimal classifier for minimizing the balanced group error. 
Beyond linear classifiers, we also employ other parameter-efficient fine-tuning paradigms (\eg, prompt tuning and adapters) to implement our debiased classifier $f_\mathsf{d}(\x)$, as detailed in~\cref{sec:ablation}. The full pipeline is outlined in~\cref{algo:all}.
\input{tables/algo}

%% file: tables/algo.tex
\begin{algorithm}[t]
\caption{Pipeline of \ours}\label{algo:all}
\begin{algorithmic}[1]
\State \textbf{Given}: Training dataset $\mathcal{D}=\{\x_i, y_i\}_{i=1}^N$, 
 class proxies $\{\z_j\}_{j=1}^K$.
\Statex \textbf{Step 1: Project}
\State Compute the projection matrix $\Pi$ using~\cref{eq:Pi}.
\State Build $f_\mathsf{b}(\x) =  W_\mathsf{b}\Pi\x$ and optimize $W_\mathsf{b}$ using~\cref{eq:la}.
\Statex \textbf{Step 2: Probe}
\State Identify pseudo groups $\{\hat{g}\}$ via~\cref{eq:pseudo}.
\State Build $h_\mathsf{d}(\x)=W_\mathsf{d}\x$ and optimize $W_\mathsf{d}$ using~\cref{eq:gla}.
\Statex \textbf{Step 3: Aggregate}
\State Construct $f_\mathsf{d}(\x)$ by aggregating weights of $W_\mathsf{d}$, as in~\cref{eq:aggregate_w}.
\State \textbf{Return} $f_\mathsf{d}(\x)$.
\end{algorithmic}
\end{algorithm}

%% file: sec/4_justify.tex
\section{Theoretical Analysis}
\subsection{Removal of Class Proxies  Amplifies Model Bias}\label{sec:train_biased}

Intuitively, removing class information forces the model $f_\mathsf{b}$  to rely more on spurious features for predicting class labels. To support this intuition, we provide theoretical evidence in a linear regression setup. Let $\vc$ denote the core features which are stable for predicting the target $y$ and $s$ be a spurious feature. The spurious feature is correlated with the target in the training set, but this correlation may not hold during testing.
 Suppose we observe $n$ features, stacked as $C=[\vc_1,\dots,\vc_n]^\top \in \mathbb{R}^{N\times d}$ and $\s=[s_1,\dots,s_n]^\top \in \mathbb{R}^{N}$. We are interested in the contribution of the spurious feature $s$ across the following two models:
\begin{itemize}
    \item \textbf{Full model:} Linear regression on the core features $\vc$ and the spurious feature $s$:
    \begin{equation}\label{eq:full_model}
    \mathbf{y}=C\bm{\beta}+\gamma \s + \bm{\varepsilon},
    \end{equation}
    where $\bm{\beta}$ and $\gamma$ are the weights associated with the core features and the spurious feature, respectively. $\bm{\varepsilon}$ is a noise term with an expected value of $0$.
    \item \textbf{Projected model:} Applying the projection matrix $\Pi$ to $C$ to obtain $\tilde{C}=C\Pi$, followed by linear regression on $\tilde{\vc}$ and $s$:
    \begin{equation}
         \mathbf{y}=\tilde{C} \tilde{\bm{\beta}} + \gamma'\s + \bm{\varepsilon}',
    \end{equation}
    where $\tilde{\bm{\beta}}$ and $\gamma'$ are the weights for the projected core features and the spurious feature, respectively.
\end{itemize}
In~\cref{eq:full_model}, the core features $C$ can be decomposed into the remaining part $\tilde{C}$ and the projected-out part $C_\mathsf{o}$.
\begin{equation}
    C=C\Pi + C(I-\Pi)=\tilde{C}+C_\mathsf{o}.
\end{equation}
Let $\mathbf{y}_\mathsf{o}=C_\mathsf{o}\bm{\beta}$ denote the contribution of the projected-out core features in the full model. Define $M=I-\tilde{C} (\tilde{C}^\top\tilde{C})^{-1}\tilde{C}^\top$, $\mathbf{r}_{\mathbf{y}_\mathsf{o}}=M\mathbf{y}_\mathsf{o}$ and $\mathbf{r}_{\s}=M\s$.
The following proposition states that projecting out core features can make the model more susceptible to spurious feature (proof in~\cref{sec:proof1}). 
\begin{proposition}\label{prop:weight}
    The weight of the spurious feature after  projection is
    \begin{equation}
        \gamma'=\gamma + \frac{\mathbf{r}_{\s}^\top  \mathbf{r}_{\mathbf{y}_\mathsf{o}}}{\mathbf{r}_{\s}^\top \mathbf{r}_{\s}}. 
    \end{equation}
\end{proposition}
\noindent\textbf{Remark.} Since the denominator $\mathbf{r}_{\s}^\top \mathbf{r}_{\s}$ is non-negative, the model relies more on spurious feature, \ie, $\gamma'>\gamma$, if $\mathbf{r}_{\s}^\top  \mathbf{r}_{\mathbf{y}_\mathsf{o}}>0$.  In other words, if the spurious feature $s$ is positively correlated with the contribution of projected-out core features $\vc_\mathsf{o}$ in the space of $M$, the weight of the spurious feature will increase. For instance, consider the Waterbirds dataset~\cite{sagawa2019distributionally}, 
the spurious feature \textit{water} background is positively correlated with the core feature \textit{waterbird} (about $95\%$ of \textit{waterbird} images have a \textit{water} background). If we remove the contribution of \textit{waterbird} by projection, the weight of the spurious feature \textit{water} will increase. 

\subsection{Group Classification and Aggregation Mitigate Spurious Correlation}\label{sec:gla}

The goal of training a debiased model to avoid spurious correlation is to minimize the Balanced Group Error ($\mathsf{BGE}$) rates~\cite{liuavoiding}:
\begin{equation}\label{eq:ber}
    \mathsf{BGE}(f)=\frac{1}{|\mathcal{G}|}\sum_{g\in \mathcal{G}}\mathbb{E}_{\x|g} [y\neq \argmax_{y'\in \mathcal{Y}}f(\x)_{y'}]
\end{equation}
A natural question is: 
what is the \textit{Bayes} optimal classifier for this problem, \ie, $f^* \in \arg\min_f \mathsf{BGE}(f)$. Suppose the underlying group-probabilities $\mathbb{P}(g|\x) \propto \exp(h^*(\x)_g)$ for unknown scorer: $h^*: \mathcal{X} \rightarrow \mathbb{R}^{|\mathcal{G}|}$. 
The \textit{Bayes} optimal classifier $f^*$ is derived from the following Proposition (proof in~\cref{sec:proof2}):
\begin{proposition}\label{propo:gla}
    Let $\mathcal{G}(y)$ denote the set of groups with class label $y$, \ie, $\mathcal{G}(y):=\{g=(y',a)\in \mathcal{G}|y'=y\}$. Let $\bm{\beta}$ denote the group priors, \ie, $\bm{\beta}_g=\mathbb{P}(g)$. The prediction :
    \begin{equation}\label{eq:bayes}
        \arg\max_{y\in \mathcal{Y}} f^*(\x)_y=\arg\max_{y\in \mathcal{Y}} \sum_{g\in \mathcal{G}(y)}(h(\x)-\ln \bm{\beta})_g
    \end{equation}
    is Bayes optimal for the problem in~\cref{eq:ber}.
\end{proposition}
\noindent\textbf{Remark.} We translate the unknown distributional class scores based on group logits $h(\x)$ and group priors $\bm{\beta}$. To obtain $f^*$, we can learn the ERM solution for group logits $h(\x)$. Then, we can minus the logarithm of the group prior $\ln\bm{\beta}$ and aggregate the logits that belongs to $\mathcal{G}(y)$ to yield the prediction for each class $y$. 
In practice, instead predict \cref{eq:bayes}, we introduce a  hyper-parameter $\tau \in \mathbb{R}_+$:
\begin{equation}
        \arg\max_{y\in \mathcal{Y}} f^*(\x)_y=\arg\max_{y\in \mathcal{Y}} \sum_{g\in \mathcal{G}(y)}(h(\x)-\tau\cdot \ln \bm{\beta})_g
\end{equation}

Inspired by Menon \etal~\cite{menonlong}, we incorporate logit adjustment into the softmax cross-entropy $\ell_\mathsf{ce}$. To do so, we directly enforce group priors offset while learning the group classifier $h_\mathsf{d}(\x)$. We refer to this as the group logit adjustment loss:
\begin{equation}
    \ell_\mathsf{gla}(\hat{g},h_\mathsf{d}(\x))= \ell_\mathsf{ce}(\hat{g},h_\mathsf{d}(\x)+\tau \cdot \ln \hat{\bm{\beta}}),
\end{equation}
which is equivalent to~\cref{eq:gla}.
Now, the estimated debiased classifier $f_\mathsf{d}(\x)$ becomes:
\begin{equation}\label{eq:ose}
    f_\mathsf{d}(\x)_y=\sum_{g\in \hat{\mathcal{G}}(y)}h_\mathsf{d}(\x)_g=\sum_{g\in \hat{\mathcal{G}}(y)} W_{\mathsf{d},g}\x
\end{equation}
Since $h_\mathsf{d}$ is linear, summing over the output-space  is equivalent to aggregating in weight-space, eliminating the overhead of group inference. 

%% file: sec/5_experiments.tex
\section{Experiments}
In this section, we evaluate the effectiveness of our method across five benchmarks: Waterbirds~\cite{sagawa2019distributionally}, CelebA~\cite{liu2015deep}, MetaShift~\cite{liangmetashift}, BAR~\cite{nam2020learning}, and Living-17~\cite{santurkar2021breeds}. For brevity, we provide an overview of our experimental setup here; further details can be found in~\cref{sec:addition-details}.

\subsection{Datasets}
\noindent\textbf{Waterbirds}~\cite{sagawa2019distributionally} was created by overlaying bird images from the Caltech-UCSD Birds~\cite{wah2011caltech} dataset onto background scenes from the Places~\cite{zhou2017places} dataset. The target classes are bird types ($\mathcal{Y}=\{\text{waterbird},\text{landbird}\}$), while the spurious attribute is the scene type ($\mathcal{A}=\{\text{water}, \text{land}\}$). Approximately $95\%$ of the waterbird training samples are associated with a water background.

\noindent\textbf{CelebA}~\cite{liu2015deep} is a real-world dataset comprising 200K celebrity portrait images. The objective is to classify hair color ($\mathcal{Y}=\{\text{not blond},\text{blond}\}$), with gender ($\mathcal{A}=\{\text{male}, \text{female}\}$) as a spurious attribute. Notably, over 94\% of the blond-haired training images feature women.

\noindent\textbf{MetaShift}~\cite{liangmetashift}: The classification objective of MetaShift is to distinguish between cats and dogs, with background type ($\mathcal{A}=\{\text{indoor}, \text{outdoor}\}$) as a spurious attribute. Following~\cite{yang2023change}, we use a pre-processed version of MetaShift. In the training set, an inherent bias emerges: cats predominantly appear in indoor scenes, while dogs are more frequently seen in outdoor settings. 

\noindent\textbf{BAR}~\cite{nam2020learning} is a real-world dataset designed to classify six actions that are each biased toward specific locations. In the original training set, the six action-location pairs are \{climbing, rock wall\}, \{diving, underwater\}, \{fishing, water surface\}, \{racing, paved track\}, \{throwing, playing Field\}, and \{vaulting, sky\}. However, the test set features different locations than the training set. We adapted the dataset for a group robustness setting by incorporating $5\%$ of the test set images into the training set.

\noindent\textbf{Living-17}~\cite{santurkar2021breeds}: 
Each class in the Living-17 represents a broad animal category that encompasses several fine-grained groups. 
The groups in the training and testing sets can exhibit noticeable visual differences, \eg, the bear class includes images of sloth bear and ice bear. Following~\cite{zhang2022contrastive}, $5\%$ of the images in each testing group are added to the training groups.

\input{tables/main_results}

\subsection{Baselines}
We compare our $\ours$ against 14 baselines: (1) Zero-shot prompting~\cite{radford2021learning}, which matches image features with classification weights by extending class names into pre-defined prompts, \eg, ``a photo of a \textit{waterbird}.''. (2) Group-informed prompting~\cite{radford2021learning}, which incorporates group descriptions into the prompts, \eg, ``a photo of a \textit{waterbird} on \textit{land}.''. (3) ERM~\cite{vapnik1991principles}, a linear classifier trained using cross-entropy loss, (4) WiSE-FT~\cite{wortsman2022robust}, which interpolates zero-shot and ERM models in weight-space, (5) Orth-Cali~\cite{chuang2023debiasing}, a debiased zero-shot model that removes biased direction in the text embedding. (6) AFR~\cite{qiu2023simple}, retrains the last layer of an ERM model with a weighted loss that emphasizes examples where the ERM
model performs poorly,  (7)
JTT~\cite{liu2021just}, (see details in~\cref{sec:commonstrategy}), (8) CnC~\cite{zhang2022correct}, which builds upon JTT~\cite{liu2021just}, but
it uses a contrastive loss to train the debiased model, (9)
CA~\cite{zhang2022contrastive}, trains contrastive adapters by learning similar representations across different groups, (10) CFR~\cite{you2024calibrating}, refines the representation by aligning them with the class centroids via contrastive loss. We also compare with three group-supervised methods -- (11) GroupDRO~\cite{sagawa2019distributionally}, (12) DFR~\cite{kirichenko2022last}, (13) S-CS~\cite{yang2023mitigating} and (14) S-SL~\cite{yang2023mitigating}.

\subsection{Training and Evaluation Details}
For each method, we evaluate both its worst-group accuracy and average (in-distribution) accuracy on the test sets. We use CLIP ResNet-50 and CLIP ViT-L/14 as primary models, while also reporting results for CLIP ResNet-101, ViT-B/32, and ViT-B/16. Following prior work~\cite{zhang2022contrastive,you2024calibrating,yang2023mitigating}, we freeze the model backbones and train all models using the SGD optimizer for 100 epochs. The hyperparameter $\tau$ is set to \{1.2, 1.0, 0.9, 1.0, 1.0\} for Waterbirds, CelebA, MetaShift, Living-17, and BAR, respectively. Model selection is based on the highest worst-group validation accuracy, except for BAR, which lacks a validation set; for BAR, we evaluate the checkpoint from the last epoch. Additionally, since the majority groups in training set do not appear in testing, we measure only the worst-group accuracy for BAR. Further implementation details are provided in~\cref{sec:addition-details}.

\input{tables/living17bar}

\subsection{Main Results}

\cref{table:results_main} presents the worst-group accuracy (WGA) and average (in-distribution) accuracy for all methods on Waterbirds~\cite{sagawa2019distributionally}, CelebA~\cite{liu2015deep} and MetaShift~\cite{liangmetashift} datasets. Compared to other group-unsupervised debiasing approaches, our $\ours$ achieves superior performance with less than $0.1\%$ trainable parameters (CLIP ResNet-50 has 25,557,032 parameters, while our model $f_\mathsf{d}(\cdot)$ has only 4,096 trainable parameters). Specifically, $\ours$ based on CLIP ResNet-50 outperforms the previous state-of-the-art, CFR~\cite{you2024calibrating}, with improvements of \{7.4\%, 17.4\%, 9.3\%\} on Waterbirds, CelebA, and MetaShift, respectively. Moreover, $\ours$ is competitive with group-supervised methods that rely on group annotations. Notably, our improvements in worst-group accuracy come at the cost of a modest reduction in in-distribution accuracy, with an average drop of 1.8 points compared to ERM models using CLIP ViT-L/14.
In~\cref{table:results_living_bar}, we report results for the Living-17~\cite{santurkar2021breeds} and BAR~\cite{nam2020learning} datasets. We again observe that $\ours$ achieves the highest worst-group performance. Notably, for the BAR dataset, where model selection is not performed and only the last checkpoint is used (due to the absence of a validation set), $\ours$ still attains the highest worst-group accuracy.

\subsection{Further Analysis and Ablation Studies}\label{sec:ablation}
We conduct comprehensive studies to (1) confirm the versatility of our proposed method—specifically, its compatibility with other efficient fine-tuning paradigms and model architectures—and (2) to understand the contribution of each component to the overall improvement.

\input{tables/other_robust_ft}
\input{tables/model_versatility}

\noindent\textbf{Extending \ours~to other parameter-efficient fine-tuning paradigms.}
In this paper, we adopt linear probes for efficient fine-tuning of foundation models. We also explore the applicability of $\ours$ to other efficient fine-tuning methods, such as prompt tuning~\cite{zhou2022learning,zhu2023prompt} (CoOp) and adapter-based tuning~\cite{gao2024clip,zhang2022contrastive,houlsby2019parameter} (small bottleneck MLPs).
Specifically, we replace the linear probe $h_\mathsf{d}(\x)=W_\mathsf{d}\x$ in Step 2 of~\cref{algo:all} with the following models:
\begin{itemize}
    \item Prompt tuning (CoOp). The prompt for group $i$ is designed as:
    \begin{equation}
        \mathbf{t}_i=[V_1^i][V_2^i]...[V_M^i][\texttt{CLASS}],
    \end{equation}
    where $[V_m^i]$ is a vector with the same dimensionality as the word embedding and serves as a parameter to be optimized. We follow CoOp to set $M=16$. [\texttt{CLASS}] represents the word embedding of the class name. Let $\Phi(\cdot)$ denote the CLIP text encoder. The group scorer for group $i$ is then defined as:
    \begin{equation}
    h_\mathsf{d}^{\mathsf{pt}}(\x;\mathbf{t})_i = \Phi(\mathbf{t}_i)^\top \x/\tau_\mathsf{clip},
    \end{equation}
    where  $\tau_\mathsf{clip}$ is the temperature.
    \item Adapter-based tuning. We adopt \cite{gao2024clip} to transform the visual embeddings $\x \in \mathbb{R}^{D}$ and the text embedding $Z \in \mathbb{R}^{K \times D}$ using simple 2-layer bottleneck MLPs. Let $H$ denote the feature the hidden-layer dimension. We follow~\cite{zhang2022contrastive} to set $H=128$. With ReLU function $\sigma(\cdot)$, visual adapter weights $W_\mathsf{v}^1\in \mathbb{R}^{H\times D}, W_\mathsf{v}^2\in \mathbb{R}^{D\times H}$ and textual adapter weights $W_\mathsf{t}^1 \in \mathbb{R}^{H\times D}, W_\mathsf{t}^2 \in \mathbb{R}^{D\times H}$, the input $\x$ and text embeddings $Z$ are adapted as:
    \begin{equation}
    \tilde{\x} = W_\mathsf{v}^2\sigma(W_\mathsf{v}^1\x), \quad \tilde{Z}^\top=W_\mathsf{t}^2\sigma(W_\mathsf{t}^1Z^\top).
    \end{equation}
    We then use the normalized adapted embeddings and the temperature $\tau_\mathsf{clip}$ to output the final scorer:
    \begin{equation}
        h_\mathsf{d}^{\mathsf{at}}(\x)_i =  \tilde{\mathbf{z}}^\top \tilde{\x}/\tau_\mathsf{clip}
    \end{equation}
\end{itemize} 

\input{tables/main_components}
\input{figures/worst_rp}
\input{figures/tau_figures}

\cref{tab:other-finetuning} demonstrates consistent improvements in worst-group accuracy over the ERM-trained CLIP ResNet-50 model across three benchmarks, highlighting the versatility of our $\ours$ across various efficient fine-tuning paradigms. For instance, on the CelebA dataset, our $\ours$ achieves worst-group accuracy improvements of 55.0\% with CoOp and 30.9\% with Adapter. Additionally, integrating PPA with these three fine-tuning methods yields comparable performance. However, linear probing is more computationally efficient than prompt tuning, as the latter requires both forward and backward passes through the large text encoder to update the prompts. It is also more storage-efficient than Adapter-based methods, which introduce additional training parameters, \ie, $\theta$ and $\phi$.

\noindent\textbf{Study on model versatility.} 
In~\cref{tab:model-versatility}, we implement our $\ours$ to various backbone architectures, including CLIP ResNet-101, CLIP ViT-B/32 and CLIP ViT-B/16. We observe consisent performance gains on the worst group accuracy, indicating the model versatility of our $\ours$.

\noindent\textbf{Main component analysis.} In~\cref{tab:main_component}, we provide an ablation study to assess the impact of the projection operation (Proj.) and group logit adjustment (GLA) loss on worst-group accuracy using CLIP ResNet-50. Row (a) represents the ERM-trained linear probing model; Row (b) probes group targets without prior correction; and Row (c) omits the projection operation during training. Row (d) shows results from our proposed method (PPA), while Row (e) replaces estimated group labels with ground-truth labels. Comparing Rows (b) and (d), we observe that adjusting the margin significantly improves worst-group accuracy. In the comparison of Rows (c) and (d), projecting out class proxies enhances group identification, resulting in better worst-group performance. Finally, Row (e) demonstrates the oracle performance using ground-truth labels, revealing that PPA approximates the maximum achievable accuracy, underscoring the importance of accurate pseudo-group labels.

\noindent\textbf{Study on the quality of estimated group labels.}
In the absence of training group labels, we rely on the quality of pseudo group labels inferred from the biased model  $f_\mathsf{b}(\mathbf{x})$ . To evaluate the quality of identified minority groups within the training set, we measure two metrics: \textit{precision}, defined as the fraction of examples within the identified minority groups that belong to the worst-performing group, and \textit{recall}, the fraction of examples from the worst-performing group captured by the identified minority groups. Here, the worst-performing group is defined as the group on which the ERM model achieves the lowest validation accuracy. In~\cref{fig:worst_rp}, we compare our method with two baselines: (1) ERM, which uses standard empirical risk minimization to train biased models, and (2) CA~\cite{zhang2022contrastive}, which leverages zero-shot models to identify minority groups. Our method, $\ours$, achieves significantly higher recall and precision than both ERM and CA. This improvement is be attributed to our biased model $f_\mathsf{b}$, which relies more heavily on spurious features to make predictions.

\noindent\textbf{Study on the hyper-parameter $\tau$.}
We examine the effect of  $\tau \in \mathbb{R}_+$  in~\cref{eq:gla}, as illustrated in~\cref{fig:tau} (more details are in~\cref{sec:more_rho}). Specifically, we vary  $\tau$  from 0.0 to 2.0 and report both the worst-group accuracy and average accuracy on the Waterbirds, CelebA, and MetaShift datasets using CLIP ResNet-50. 
A larger  $\tau$ promotes a greater margin for the minority groups, while $\tau=0$ reduces to the standard cross-entropy loss.
Our findings reveal that (1) average accuracy decreases monotonically as  $\tau$  increases, and (2) the optimal worst-group accuracy is generally achieved around  $\tau = 1.0$. For example, the optimal values of $\tau$ are $1.2, 1.0$ and $0.9$ for Waterbirds, CelebA and MetaShift, respectively. This aligns with our theoretical analysis, which suggests that the \textit{Bayes} optimal classifier for minimizing the balanced group error is attained when $\tau = 1$.  


%% file: tables/main_results.tex
\begin{table*}[ht]
\begin{center}
    \caption{\textbf{Evaluation of methods for improving group robustness of CLIP models across the Waterbirds, CelebA, and MetaShift benchmarks.} Best worst-group accuracy (WGA) of the methods without group labels are in {\textbf{bold}}.}
\label{table:results_main}
\begin{adjustbox}{width=0.95\linewidth}
\begin{tabular}{cccccccccccccc}
	\toprule
	& & \multicolumn{6}{c}{\textbf{CLIP ResNet-50}} & \multicolumn{6}{c}{\textbf{CLIP ViT-L/14}}\\
	\cmidrule(r){3-8} \cmidrule(r){9-14}
 Group labels & & \multicolumn{2}{c}{Waterbirds} & \multicolumn{2}{c}{CelebA} & \multicolumn{2}{c}{MetaShift} & \multicolumn{2}{c}{Waterbirds} & \multicolumn{2}{c}{CelebA} & \multicolumn{2}{c}{MetaShift} \\ 
    \cmidrule(r){3-4} \cmidrule(r){5-6} \cmidrule(r){7-8} \cmidrule(r){9-10} \cmidrule(r){11-12} \cmidrule(r){13-14} 
in train sets?			&{Method}
		            & WGA
		            & Avg
		            & WGA
		            & Avg
		            & WGA
		         
		            & Avg
		            & WGA
		            & Avg
		            & WGA
		            & Avg
                  
                    & WGA
                    & Avg
		            \\ \midrule
        \multirow{4}{*}{\ding{51}}
        &{GroupDRO~\cite{sagawa2019distributionally}}
                    & {75.1}
                    & {83.8}
                    & {84.1}
		            & {89.5}
                    
                    & {83.2}
                    & {87.3}
                    & {90.8}
		            & {96.4}
                    & {88.3}
                    & {91.2}
                   
                    & {93.9}
                    & {97.4}
                    \\
        &{S-CS~\cite{yang2023mitigating}}
                    & {77.5}
                    & {83.2}
                    & {75.2}
		            & {80.4}         
                    & {81.2}
                    & {89.8}
                    & {89.1}
		            & {95.7}
                    & {86.1}
                    & {89.3}
                    & {92.3}
                    & {97.1}
                    \\
        &{S-CL~\cite{yang2023mitigating}}
                    & 75.2
                    & 86.0
                    & 75.6
		            & 80.4
                    & 81.5
                    & 88.8
                    & 89.9
		            & 96.0
                    & 87.8
                    & 90.5
                    & 93.1
                    & 96.9
                    \\
        &{DFR~\cite{kirichenko2022last}}
                    & {73.2}
                    & {83.8}
                  
                    & {80.0}
                    & {92.8}
                    & {83.1}
                    & {88.3}
                    & {89.7}
		            & {97.8}
                    & {85.6}
                    & {90.8}
                   
                    & {92.3}
                    & {97.0}
                    \\\midrule 
        \multirow{12}{*}{\ding{55}}
        & Zero-Shot (ZS)~\cite{radford2021learning}
                    & {54.2}
                    & {92.4}
                    & {55.0}
		            & {88.0}
                   
                    & {86.2}
                    & {95.4}
                    & {26.5}
		            & {88.2}
                    & {27.0}
                    & {85.9}
                    
                    & {93.2}
                    & {96.2}
                    \\
        & Group Prompt ZS~\cite{radford2021learning}
                    & {46.4}
                    & {91.7}
                    & {53.4}
		            & {73.5}
                   
                    & {84.6}
                    & {95.2}
                    & {25.4}
		            & {85.8}
                    & {66.9}
                    & {83.1}
                  
                    & {93.9}
                    & {96.7}
                    \\
        & ERM~\cite{vapnik1991principles}
                    & {7.9}
                    & {93.5}
                    & {11.9}
		            & {94.7}
                
                    & {75.4}
                    & {94.4}
                    & {65.9}
		            & {97.6}
                    & {28.3}
                    & {94.7}
                    
                    & {84.6}
                    & {96.7}
                    \\
        & WiSE-FT~\cite{wortsman2022robust}
                    & {49.8}
                    & {91.0}
                    & {85.6}
		            & {88.6}
                   
                    & {86.2}
                    & {95.4}
                    & {65.9}
		            & {97.6}
                    & {80.0}
                    & {87.4}
                    
                    & {93.9}
                    & {97.2}
                    \\
        & Orth-Cali~\cite{chuang2023debiasing}
                    & {74.0}
                    & {78.7}
                    & {82.2}
		            & {84.4}
                  
                    & {86.2}
                    & {94.8}
                    & {68.8}
		            & {84.5}
                    & {76.1}
                    & {86.2}
                   
                    & {92.7}
                    & {96.2}
                    \\
        &{AFR~\cite{qiu2023simple}}
                    & {48.4}
                    & {89.3}
                    & {53.4}
		            & {94.3}
                  
                    & {76.9}
                    & {86.8}
                    & {73.4}
		            & {88.2}
                    & {70.0}
                    & {85.2}
                  
                    & {90.3}
                    & {97.1}
                    \\
        &{JTT~\cite{liu2021just}}
                    & {61.7}
                    & {90.6}
                    & {60.2}
		            & {79.9}
                 
                    & {78.5}
                    & {89.4}
                    & {83.6}
		            & {97.3}
                    & {75.6}
                    & {93.3}
                 
                    & {91.2}
                    & {94.2}
                    \\
        &{CnC~\cite{zhang2022correct}}
                    & {61.2}
                    & {87.1}
                    & {63.9}
		            & {90.3}
               
                    & {78.3}
                    & {87.1}
                    & {84.5}
		            & {97.5}
                    & {79.2}
                    & {89.3}
             
                    & {92.2}
                    & {94.7}
                    \\
     &{CA~\cite{zhang2022contrastive}}
                    & 83.7
                    & 89.4
                    & 90.0
		            & 90.7
                    & 77.9
                    & 85.5
                    & 86.9
		            & 96.2
                    & 84.6
                    & 90.4
                    & 91.3
                    & 93.4
                    \\
        & CFR~\cite{you2024calibrating}
                    &  76.9
                    & 77.6
                    & 73.7
                    &  81.1
                    &  81.5
                    &  89.5
                    &  \textbf{88.2}
		            &  96.8
                    &  84.8
                    &  87.8
                    &  93.7
                    &  95.5
                    \\ 
    \rowcolor{tabhighlight}
   \cellcolor{white}        & $\ours$ (ours)
                    &  \textbf{84.3}
                    &  88.3
                    &  \textbf{91.1}
                    &  92.1
                    &  \textbf{90.8}
                    &  94.7
                    &  87.2
		            &  94.6
                    &  \textbf{90.4}
                    &  91.0
                    &  \textbf{94.8}
                    &  96.8
                    \\ 
	\bottomrule
	\end{tabular}
    \end{adjustbox}
    \end{center}
\end{table*}

%% file: tables/living17bar.tex
\begin{table}[t]
\begin{center}
    \caption{Experimental results on Living-17 and BAR benchmarks.}
\label{table:results_living_bar}
\begin{adjustbox}{width=0.95\linewidth}
\begin{tabular}{ccccccc}
	\toprule
	 & \multicolumn{3}{c}{\textbf{CLIP ResNet-50}} & \multicolumn{3}{c}{\textbf{CLIP ViT-L/14}}\\
	\cmidrule(r){2-4} \cmidrule(r){5-7}
  & \multicolumn{2}{c}{Living-17} & \multicolumn{1}{c}{BAR}  & \multicolumn{2}{c}{Living-17} & \multicolumn{1}{c}{BAR}  \\ 
    \cmidrule(r){2-3} \cmidrule(r){4-4} \cmidrule(r){5-6} \cmidrule(r){7-7} 
		{Method} 
		            & WGA
		            & Avg
		            & WGA
		         
		            & WGA
		            & Avg
		            & WGA
                  
		            \\ \midrule
         ZS~\cite{radford2021learning}
                    & 38.0
                    & 83.2
                    & 74.4
		            & 48.0
                    & 92.3
		            & 81.7
                    \\
         ERM~\cite{vapnik1991principles}
                    & 53.3
                    & 90.8                
                    & 54.9
                    & 84.0
                    & 98.6
		            & 89.0
                    \\
        {JTT~\cite{liu2021just}}

                    & 44.0
                    & 86.4
                    & 58.5
		            & 78.7
                    & 97.3
                    & 90.2

                    \\
     {CA~\cite{zhang2022contrastive}}
                    & 62.0
                    & 90.9
                    & -
		            & 80.0
                    & 97.5
                    & -
                    \\
    \rowcolor{tabhighlight}
   $\ours$ (ours)
           
                    &  \textbf{65.0}
                    &  92.5
                    &  \textbf{76.8}
                    &  \textbf{85.0}
                    &  98.1
		            &  \textbf{91.5}
                    \\ 
	\bottomrule
	\end{tabular}
    \end{adjustbox}
    \end{center}
\end{table}

%% file: tables/other_robust_ft.tex
\begin{table}[t]
\caption{$\ours$ consistently improves group robustness for across different efficient fine-tuning paradigms using CLIP ResNet-50. The results using CLIP ViT-L/14 are reported in~\cref{sec:versa-fine-tuning}.}\label{tab:other-finetuning}
\begin{adjustbox}{width=1.\linewidth}
\begin{tabular}{lcccccc}
\toprule
\multirow{2}{*}{Method} & \multicolumn{2}{c}{Waterbirds} & \multicolumn{2}{c}{CelebA} &  \multicolumn{2}{c}{MetaShift} \\
    \cmidrule(r){2-3} \cmidrule(r){4-5} \cmidrule(r){6-7} 
                        & WGA            & Avg           & WGA          & Avg         & WGA                   & Avg           \\ \midrule
Linear Probe + ERM            &  7.9  
& 93.5
& 11.9 
& 94.7
& 75.4
& 94.4              \\
\rowcolor{tabhighlight}
Linear Probe + PPA             & 84.3               
& 88.3
& 91.1
& 92.1
& 90.8
& 94.7        \\ \midrule   
CoOp + ERM     
& 61.7
& 96.2
& 35.0
& 94.8
& 75.4
& 92.8            \\
\rowcolor{tabhighlight}
CoOp + PPA       
& 83.0
& 87.6
& 90.0
& 92.2
& 89.0
& 93.3
\\ \midrule
Adapter + ERM     
& 61.4
& 96.7
& 60.6
& 94.4
& 84.6
& 93.9
\\
\rowcolor{tabhighlight}
Adapter + PPA               
&  78.2            
&  91.8          
&  91.5          
&  92.5          
&  87.7            
&  93.4 
\\ \bottomrule
\end{tabular}
\end{adjustbox}
\end{table}

%% file: tables/model_versatility.tex
\begin{table}[t]
\caption{$\ours$ consistently improves group robustness on Waterbirds across different architectures.}\label{tab:model-versatility}
\begin{adjustbox}{width=1.\linewidth}
\begin{tabular}{lcccccc}
\toprule
\multirow{2}{*}{Method} & \multicolumn{2}{c}{CLIP RN-101} & \multicolumn{2}{c}{CLIP ViT-B/32} &  \multicolumn{2}{c}{CLIP ViT-B/16} \\
    \cmidrule(r){2-3} \cmidrule(r){4-5} \cmidrule(r){6-7} 
                        & WGA            & Avg           & WGA          & Avg         & WGA                   & Avg           \\ \midrule
Zero-shot           
&  33.6  
&  90.0
& 47.0
& 88.8
& 34.0
& 88.1
\\
ERM LP          
& 62.9              
& 96.5
& 56.8
& 95.6
& 66.5
& 96.4 
\\ 
\rowcolor{tabhighlight}
PPA     
& 83.6
& 89.7
& 83.5
& 89.1
& 84.4
& 89.9 
\\
\bottomrule
\end{tabular}
\end{adjustbox}
\end{table}

%% file: tables/main_components.tex
\begin{table}[t]
\centering
\caption{\textbf{Main component analysis.} We present the worst group accuracies using CLIP ResNet-50. ``\textbf{Proj.}'' and ``\textbf{GLA}'' stands for the projection operation in~\cref{eq:biased-classifier} and group logit adjustment loss in~\cref{eq:gla}, respectively. ``\textbf{GT}'' means we use the ground-truth group labels for training debiased models.}
\label{tab:main_component}
\begin{adjustbox}{width=1.0\linewidth}
\begin{tabular}{cccc|ccc}
\toprule
& \textbf{Proj.} & \textbf{GLA} & \textbf{GT} & Waterbirds & CelebA & MetaShift \\ \midrule
 (a)&     &     &  &  7.9        & 11.9     & 75.4          \\
(b)&\checkmark      &  &    & 54.4 & 29.4      &  86.2        \\
(c) && \checkmark & & 81.6 & 70.0 & 89.2 \\ 
\rowcolor{tabhighlight}
(d)& \checkmark & \checkmark & & 84.3  &91.1 & 90.8 \\ \midrule
 (e)& &\checkmark & \checkmark & 86.8 & 91.5& 91.3 \\  \bottomrule
\end{tabular}
\end{adjustbox}
\end{table}

%% file: figures/worst_rp.tex
\begin{figure}[t]
    \centering
    \includegraphics[width=0.47\textwidth]{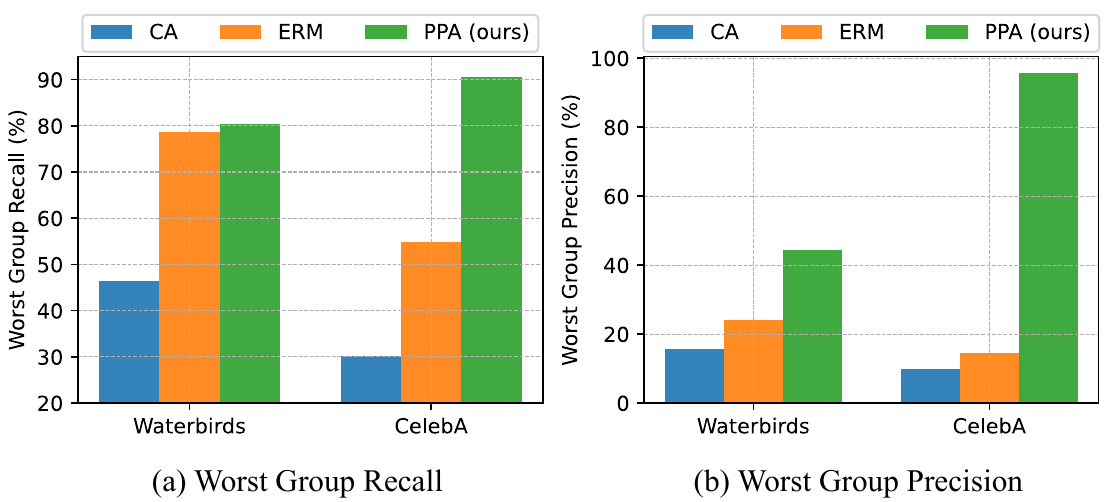}
    \caption{The precision and recall of the worst-group examples on Waterbirds and CelebA using CLIP ResNet-50.}
    \label{fig:worst_rp}
\end{figure}

%% file: figures/tau_figures.tex
\begin{figure*}[t]
    \centering
    \includegraphics[width=0.95\textwidth]{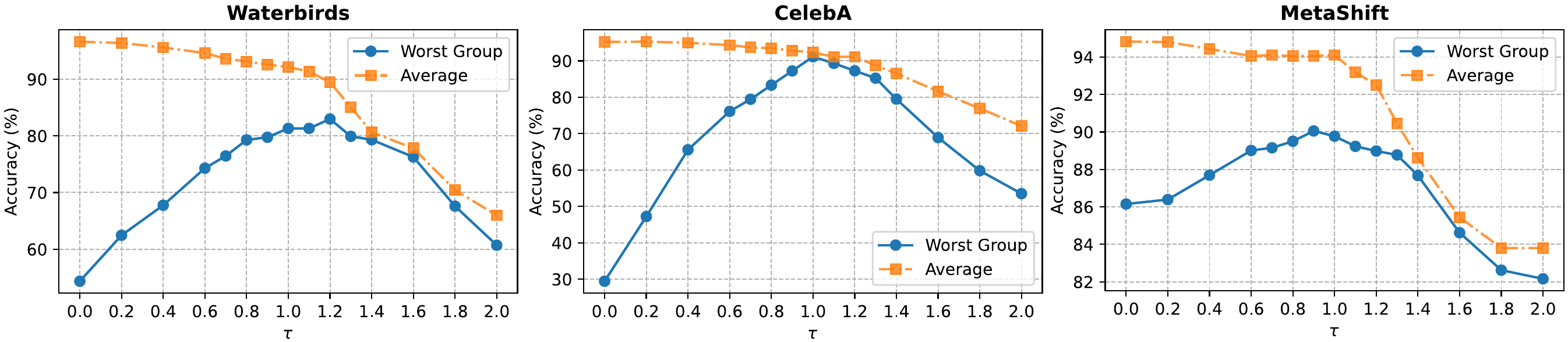}
    \caption{\textbf{Study on the effect of $\tau$ using CLIP ResNet-50.} The optimal worst group accuracy is typically observed around $\tau=1.0$.}
    \label{fig:tau}
\end{figure*}

%% file: sec/6_conclusion.tex
\section{Conclusion}

In this paper, we introduce Project-Probe-Aggregate ($\ours$), a novel, parameter-efficient fine-tuning framework for enhancing the group robustness of image-text foundation models without relying on training group labels. Our approach improves the identification of minority groups by projecting image features onto the null space of class proxies. Additionally, we design a group logit adjustment loss to minimize balanced group error and mitigate spurious correlations. Through theoretical analysis and empirical evaluations across diverse benchmarks, fine-tuning paradigms and model architectures, we demonstrate that the proposed method effectively enhances robust group performance of foundation models.

\section*{Acknowledgment}
This research is supported by the RIE2025 Industry Alignment Fund – Industry Collaboration Projects (IAF-ICP) (Award I2301E0026), administered by A*STAR, as well as supported by Alibaba Group and NTU Singapore through
Alibaba-NTU Global e-Sustainability CorpLab (ANGEL).

%% file: appendix/1_proofs.tex
\newpage

\section{Proofs}

\subsection{Proof of~\cref{prop:weight}}\label{sec:proof1}
\begin{restatedproposition}[\cref{prop:weight}]
     The weight of the spurious feature after  projection is
    \begin{equation}
        \gamma'=\gamma + \frac{\mathbf{r}_{\s}^\top  \mathbf{r}_{\mathbf{y}_\mathsf{o}}}{\mathbf{r}_{\s}^\top \mathbf{r}_{\s}}. 
    \end{equation}
\end{restatedproposition}

\begin{proof}
    
We have $d \in \mathbb{R}^d$ core features $\vc$ which determine the prediction target $y$ and the spurious features $s$. 
Suppose we observe $n$ features, stacked as 
$C=[\vc_1,\cdots,\vc_n] \in \mathbb{R}^{N\times d}$ and $\s=[s_1,\cdots,s_n] \in \mathbb{R}^N$. 
We have two regression scenarios:
\begin{itemize}
    \item \textbf{Full model:} Linear regression on the core features $\z$ and the spurious feature $s$:
    \begin{equation}\label{eq:sc1}
        \mathbf{y}=C\bm{\beta} + \gamma\s + \bm{\varepsilon},
    \end{equation}
    where $\bm{\beta}$ and ${\gamma}$ are the weights associated with the core features $\vc$ and the spurious feature $s$, respectively. $\bm{\varepsilon}$ is a noise term with an expected value of 0.
    \item \textbf{Projected model:} Applying the projection matrix $\Pi$ to $C$ to obtain $\tilde{C}=C\Pi$, followed by linear regression on $\tilde{\vc}$ and $s$:
    \begin{equation}\label{eq:sc2}
        \mathbf{y}=\tilde{C} \tilde{\bm{\beta}} + \gamma'\s + \bm{\varepsilon}'
    \end{equation}
    where $\tilde{\bm{\beta}}$ and ${\gamma}'$ are the weights for the projected core features $\tilde{\vc}$ and the spurious feature $s$. 
\end{itemize}
We define $M=I-\tilde{C} (\tilde{C}^\top\tilde{C})^{-1}\tilde{C}^\top$, $\mathbf{r}_\mathbf{y}=M\mathbf{y}$ and $\mathbf{r}_\mathbf{s}=M\mathbf{s}$. Applying $M$ to both sides of~\cref{eq:sc2}, we obtain:
\begin{align}
\mathbf{r}_\mathbf{y} &= M\tilde{C} \tilde{\bm{\beta}} + \gamma'\mathbf{r}_\mathbf{s}  + M \bm{\varepsilon}' \\
&= \gamma'\mathbf{r}_\mathbf{s} + M \bm{\varepsilon}' && \text{[$M\tilde{C}=0$]}
\end{align}
The weight of spurious feature is derived as:
\begin{equation}\label{eq:gamma1}
    {\gamma}'=(\mathbf{r}_\mathbf{s}^\top \mathbf{r}_\mathbf{s})^{-1}\mathbf{r}_\mathbf{s}^\top \mathbf{r}_\mathbf{y}.
\end{equation}
The core features  $C$ can be decomposed into the remaining part $\tilde{C}$ and the projected-out part $C_\mathsf{o}$:
\begin{equation}\label{eq:decompose}
    C=C\Pi+C(1-\Pi)=\tilde{C} + C_\mathsf{o}
\end{equation}
Combine~\cref{eq:decompose} and~\cref{eq:sc1}, we have:
\begin{equation}
        \mathbf{y}=\tilde{C}\bm{\beta}+ C_\mathsf{o}\bm{\beta} + {\gamma}\s + \bm{\varepsilon}
\end{equation}
Denote $\mathbf{y}_\mathsf{o}=C_\mathsf{o}\bm{\beta}$ is contribution of the projected-out core features and $\mathbf{r}_{\mathbf{y}_\mathsf{o}}=M\mathbf{y}_\mathsf{o}$, we can express $\mathbf{r}_\mathbf{y}$ as:
\begin{align}
\mathbf{r}_\mathbf{y}&=M\mathbf{y} \\
&=M(\tilde{C}\bm{\beta}+ C_\mathsf{o}\bm{\beta} + \gamma\s + \bm{\varepsilon}) \\
&=\mathbf{r}_{\mathbf{y}_\mathsf{o}} + \gamma\mathbf{r}_\mathbf{s} + M\bm{\varepsilon} && \text{[$M\tilde{C}=0$]}
\end{align}
Plugging into~\cref{eq:gamma1} and omitting the noise term (since $\mathbb{E}[\varepsilon\cdot s]=0$), we have:
\begin{align}
    \gamma'&=(\mathbf{r}_\mathbf{s}^\top \mathbf{r}_\mathbf{s})^{-1} \mathbf{r}_\mathbf{s}^\top(\mathbf{r}_\mathbf{s}\gamma + \mathbf{r}_{\mathbf{y}_\mathsf{o}}) \\
    &=\gamma + (\mathbf{r}_\mathbf{s}^\top \mathbf{r}_\mathbf{s})^{-1}\mathbf{r}_\mathbf{s}^\top \mathbf{r}_{\mathbf{y}_\mathsf{o}}
\end{align}

\end{proof}

\subsection{Proof of~\cref{propo:gla}}\label{sec:proof2}

\begin{lemma}\label{lemma:ber}
    Let $\hat{y}=\argmax_{y'\in \mathcal{Y}}f(\x)_{y'}$ denote the prediction of $f$. The balanced group error ($\mathsf{BGE}$) defined in~\cref{eq:ber} can be expressed as:
    \begin{equation}\label{eq:lemma}
        \mathsf{BGE}(f)=\frac{1}{|\mathcal{G}|}\mathbb{E}_{\x}[\sum_{g\in \mathcal{G}}
        \frac{\mathbb{P}(g|\x)}{\mathbb{P}(g)}\cdot \mathbb{P}(y\neq \hat{y}|\x)].
    \end{equation}
\end{lemma}

\begin{proof}
    \begin{align}
        \mathsf{BGE}(f)&=\frac{1}{|\mathcal{G}|}\sum_{g\in \mathcal{G}}\mathbb{E}_{\x|g} [y\neq \argmax_{y'\in \mathcal{Y}}f(\x)_{y'}] \\
        &=\frac{1}{|\mathcal{G}|}\sum_{g\in \mathcal{G}}\int_{\x}\mathds{1}[y\neq \hat{y}]\mathbb{P}(\x|g) \text{d}\x \\
        &=\frac{1}{|\mathcal{G}|}\sum_{g\in \mathcal{G}}\int_{\x}\mathds{1}[y\neq \hat{y}] \frac{\mathbb{P}(g|\x)}{\mathbb{P}(g)}\mathbb{P}(\x) \text{d}\x \\ 
        &=\frac{1}{|\mathcal{G}|}\sum_{g\in \mathcal{G}} \mathbb{E}_\x[\frac{\mathbb{P}(g|\x)}{\mathbb{P}(g)}\cdot \mathds{1}[y\neq \hat{y}]] \\
        &=\frac{1}{|\mathcal{G}|}\mathbb{E}_\x[\sum_{g\in \mathcal{G}} \frac{\mathbb{P}(g|\x)}{\mathbb{P}(g)}\cdot \mathds{1}[y\neq \hat{y}]] \label{eq:ber_mid}
    \end{align} 
    The expected value $\mathbb{E}_\x[\mathds{1}[y\neq \hat{y}]]$ can be expressed as the joint expectation over $\x$ and $y$.  By applying the law of total expectation, we rewrite~\cref{eq:exp_value} by conditioning on $\x$ in~\cref{eq:condition_rewrite}.
    \begin{align}
    \mathbb{E}_\x[\mathds{1}[y\neq \hat{y}]]&=\mathbb{E}_{\x,y}[\mathds{1}[y\neq \hat{y}]] \label{eq:exp_value} \\
        &=\mathbb{E}_\x\mathbb{E}_{y|\x}[\mathds{1}[y=\hat{y}|\x]] \label{eq:condition_rewrite}
    \end{align}
    Given $\x$, $\hat{y}$ is deterministic. Thus the inner expectation  simplifies to the probability of  $y=\hat{y}$ conditioned on $\x$:
  \begin{equation}
      \mathbb{E}_{y|\x}[\mathds{1}[y=\hat{y}|\x]=\mathbb{P}(y=\hat{y}|\x)
  \end{equation}
    Substituting back into~\cref{eq:condition_rewrite}, we have:
    \begin{equation}\label{eq:error_final}
         \mathbb{E}_\x[\mathds{1}[y\neq \hat{y}]]=\mathbb{E}_\x[\mathbb{P}(y=\hat{y}|\x))]
    \end{equation}
    Combining~\cref{eq:error_final} with~\cref{eq:ber_mid}, we arrive at~\cref{eq:lemma}.
\end{proof}
\begin{restatedproposition}[\cref{propo:gla}
]
     Let $\mathcal{G}(y)$ denote the set of groups with class label $y$, \ie, $\mathcal{G}(y):=\{g=(y',a)\in \mathcal{G}|y'=y\}$. Let $\bm{\beta}$ denote the group priors, \ie, $\bm{\beta}_g=\mathbb{P}(g)$. The prediction :
    \begin{equation}\label{eq:bayes}
        \arg\max_{y\in \mathcal{Y}} f^*(\x)_y=\arg\max_{y\in \mathcal{Y}} \sum_{g\in \mathcal{G}(y)}(h(\x)-\ln \bm{\beta})_g
    \end{equation}
    is Bayes optimal for the problem in~\cref{eq:ber}.
\end{restatedproposition}

\begin{proof}
    Using \cref{lemma:ber}, to minimize the balanced group error, it is equivalent to minimize the term inside the expectation:
    \begin{align}
        &\sum_{g\in \mathcal{G}}\frac{\mathbb{P}(g|\x)}{\mathbb{P}(g)}\cdot \mathbb{P}(y\neq \hat{y}|\x) \\
        &=\sum_{y\in \mathcal{Y}}[\sum_{g \in \mathcal{G}(y)}\frac{\mathbb{P}(g|\x)}{\mathbb{P}(g)}\cdot (1-\mathbb{P}(y=\hat{y}|\x))]
    \end{align}
    It is equivalent to maximize:
    \begin{equation}
        \sum_{y\in \mathcal{Y}}[\sum_{g \in \mathcal{G}(y)}\frac{\mathbb{P}(g|\x)}{\mathbb{P}(g)}]\cdot \mathbb{P}(y=\hat{y}|\x).
    \end{equation}
    Denote $a_y=\sum_{g \in \mathcal{G}(y)}\frac{\mathbb{P}(g|\x)}{\mathbb{P}(g)}$ and $b_y=\mathbb{P}(y=\hat{y}|\x)$. Since $\{a_y\}$ are fixed and $b_y$ is a probability simplex, \ie, $b_y>0$ and $\sum_y b_y=1$. We are equivalent to solving the following constrained optimization problem:
    \begin{equation}
        \max_{b_y} \sum_y a_y b_y,\ \text{s.t.}\ b_y>0, \sum_y b_y=1
    \end{equation}
    It is straightforward to show that the optimal value is $\max_i a_i$, achieved when $i = \argmax_i a_i$, with $b_i = 1$ and $b_j = 0$ for $j \neq i$. Substituting back the definitions of $a$ and $b$.
    The solution is:
    \begin{equation}
    \mathbb{P}(y=\hat{y}|\x)=\begin{cases} 
   1, & \text{if } y=\arg\max_{y'}\sum_{g \in \mathcal{G}(y')}\frac{\mathbb{P}(g|\x)}{\mathbb{P}(g)} \\
    0, & \text{otherwise} 
    \end{cases}
\end{equation}
It is equivalent to show that:
\begin{equation}
    \mathbb{P}(\arg\max_{y'}\sum_{g \in \mathcal{G}(y')}\frac{\mathbb{P}(g|\x)}{\mathbb{P}(g)}=\hat{y}|\x)=1
\end{equation}
Therefore, the \textit{Bayes} optimal solution is:
\begin{equation}\label{eq:bayesv1}
    \arg\max_{y\in \mathcal{Y}} f^*(\x)_y=\arg\max_{y'}\sum_{g \in \mathcal{G}(y')}\frac{\mathbb{P}(g|\x)}{\mathbb{P}(g)}
\end{equation}
With the definition $\bm{\beta}_g=\mathbb{P}(g)$ and $\mathbb{P}(g|\x)\propto \exp(h(\x)_g)$, we have:
\begin{equation}\label{eq:bayesv2}
    \frac{\mathbb{P}(g|\x)}{\mathbb{P}(g)}\propto \frac{\exp(h(\x)_g)}{\exp(\ln\bm{\beta}_g)} \propto (h(\x) - \ln\bm{\beta})_g
\end{equation}
Since proportional scaling does not change argmax results, combining~\cref{eq:bayesv2} with~\cref{eq:bayesv1}, we obtain~\cref{eq:bayes}.
\end{proof}

%% file: appendix/3_additional_experimental_details.tex
\section{Additional Experimental Details}\label{sec:addition-details}
\input{tables/prompt_templates}
\input{tables/other-ft-vit14}

\subsection{Additional Implementation Details}
For all methods evaluated in our experiments, we use the SGD optimizer with a weight decay of $5\times 10 ^{-5}$ and a momentum of $0.9$. The initial learning rate is set to 0.0002 and decreases to 0 using cosine annealing. The models are trained for 100 epochs, with a warm-up learning rate of $10^{-5}$ applied during the first epoch to mitigate explosive gradients in the early training iterations. The batch size is set to 128 for most datasets, except for CelebA, where it is increased to 512 to accelerate training due to the dataset’s relatively larger size. All classification heads including linear probing, prompt tuning and adapters are initialized with the zero-shot prompting. 
For all datasets except BAR, we evaluate the model on the validation set at the end of each epoch and select the one with the highest worst-group accuracy for final testing. For the BAR dataset, which lacks a validation set, we use the checkpoint from the last epoch for testing. The hyperparameter $\tau$ is searched within the range $[0.8, 0.9, 1.0, 1.1, 1.2]$. All experiments are conducted in a single NVIDIA A6000 GPU.

\subsection{Prompt Templates}
In~\cref{table:prompt_templates}, we present the prompt templates for zero-shot prompting and group-informed prompting for each dataset. Zero-shot prompting with class names is also used to construct the class proxy matrix $Z$ in step 1 of our $\ours$. \cref{table:class_group_names} lists the class names and group names for all datasets.

\subsection{Versatility of Fine-Tuning Paradigms}\label{sec:versa-fine-tuning}

In~\cref{tab:other-ft-vit14}, we apply our $\ours$ to other parameter-efficient fine-tuning paradigms using CLIP ViT-L/14 models. We observe consistent gains in worst group accuracies.

 We further extend our method to train 2-Layer MLP after CLIP, with results in~\cref{tab:other-ft}, showing that the 2-Layer MLP offers no significant gains over the linear layer.

 \begin{table}[t]
\centering
\caption{Results of other fine-Tuning paradigms.}\label{tab:other-ft}
\scalebox{0.8}{
\begin{tabular}{lcccccc}
\toprule
\multirow{2}{*}{} & \multicolumn{2}{c}{Waterbirds} & \multicolumn{2}{c}{CelebA} &  \multicolumn{2}{c}{MetaShift} \\
    \cmidrule(r){2-3} \cmidrule(r){4-5} \cmidrule(r){6-7} 
                        & WGA            & Avg           & WGA          & Avg         & WGA                   & Avg           \\ \midrule
Linear Layer                           
& 84.3
& 88.3
& 91.1
& 92.1
& 90.8
& 94.7     \\ 
2-Layer MLP                           
& 83.4
& 88.1
& 90.4
& 92.5
& 90.1
& 95.9   \\ 
Full Fine-Tuning                           
& 83.7
& 89.8
& 91.8
& 93.2
& 89.5
& 95.2     \\
\bottomrule
\end{tabular}
}
\end{table}

\subsection{Noise Sensitivity of Pseudo-Labels} 
To assess the noise sensitivity, we randomly select $p \%$  of the training samples and assign random values to subgroup labels within each class to introduce pseudo-label errors. 
The worst-group accuracies for varying $p$ are shown in the figure below. Our results indicate that the proposed method maintains high WGA when label noise is below $10\%$, demonstrating its robustness under mild noise conditions.

\begin{figure}[t]
  \centering
  \includegraphics[width=0.8\linewidth]{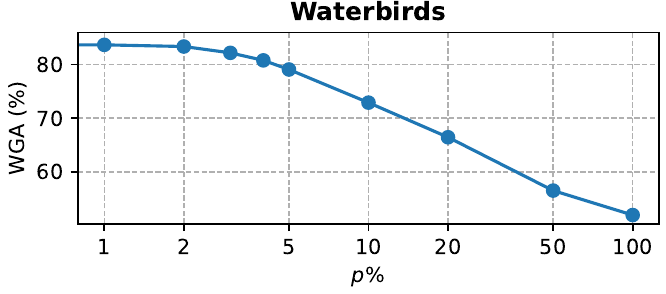}
  \caption{Results on noise sensitivity.}
\end{figure}

\subsection{Comparison with Other Methods}
We compare our model with~\cite{kim2024discovering} using CLIP ResNet-50 in~\cref{tab:other_methods}. 

\begin{table}[h]
\centering
\caption{Comparison with other methods.}
\label{tab:other_methods}
\scalebox{0.8}{
\begin{tabular}{lcccc}
\toprule
\multirow{2}{*}{Method} & \multicolumn{2}{c}{Waterbirds} & \multicolumn{2}{c}{CelebA} \\
    \cmidrule(r){2-3} \cmidrule(r){4-5} 
                        & WGA            & Avg           & WGA          & Avg                  \\ \midrule
CLIP + B2T~[2]                          
& 61.7
& 76.9
& 80.0
& 87.2    \\  
CLIP + PPA~(ours)
& 84.3
& 88.3
& 91.1
& 92.1 \\
\bottomrule
\end{tabular}
}
\end{table}

\subsection{Results without Tuning $\tau$}
\label{sec:more_rho}The results with $\tau = 1$ are reported in~\cref{sec:details_rho}. As expected, $\tau=1$ still achieves SOTA.
\begin{table}[h]
\centering
\caption{Results without Tuning $\tau$.}
\label{sec:details_rho}
\scalebox{0.8}{
\begin{tabular}{lcccccc}
\toprule
\multirow{2}{*}{} & \multicolumn{2}{c}{Waterbirds} & \multicolumn{2}{c}{CelebA} &  \multicolumn{2}{c}{MetaShift} \\
    \cmidrule(r){2-3} \cmidrule(r){4-5} \cmidrule(r){6-7} 
                        & WGA            & Avg           & WGA          & Avg         & WGA                   & Avg           \\ \midrule
Optimal $\tau$            
&  84.3
& 88.3
& 91.1
& 92.1
& 90.8
& 94.7         \\
$\tau=1$                           
& 82.7
& 91.3
& 91.1
& 92.1
& 89.8
& 94.1     \\  \bottomrule
\end{tabular}
}
\end{table}

%% file: tables/prompt_templates.tex
\begin{table*}[t]
\begin{center}
\caption{Class prompt and group prompt templates.}
\label{table:prompt_templates}
\begin{tabular}{lll}
\toprule
Dataset    & Class Prompt & Group Prompt \\ \midrule
Waterbirds & \textit{a type of bird, a photo of a \{class\}.}  &          \textit{a type of bird, a photo of a \{class\} on \{group\}.}      \\
CelebA     & \textit{a photo of a celebrity with \{class\}.}             &   \textit{a photo of a celebrity, a \{group\} with \{class\}}.           \\
MetaShift  &    \textit{a photo of a \{class\}.}             &           \textit{a photo of a \{group\} \{class\}.}     \\ 
BAR &a photo of a person doing \{class\}. & N/A \\
Living-17  &   \textit{a photo of a \{class\}.}           &           N/A   \\
\bottomrule
\end{tabular}
\end{center}
\end{table*}

\begin{table*}[t]
\begin{center}
\caption{Class and group names.}
\label{table:class_group_names}
\begin{tabular}{lll}
\toprule
Dataset    & Class Names & Group Names \\ \midrule
Waterbirds & landbird, waterbird &    land, water       \\
CelebA     &  non-blond hair, blond hair   &  man, woman            \\
MetaShift  &  dog, cat         & outdoor, indoor     \\ 
BAR & climbing, diving, fishing, pole vaulting, racing, throwing & N/A \\
\multirow{2}{*}{Living-17}  & \multirow{2}{*}{\begin{tabular}[c]{@{}l@{}}salamander, turtle,  lizard, snake, spider, grouse, parrot,
crab, \\ dog, wolf, fox, cat, bear, beetle, butterfly, ape,  monkey \end{tabular}}  &          \multirow{2}{*}{N/A }   \\
&& \\
\bottomrule
\end{tabular}
\end{center}
\end{table*}

%% file: tables/other-ft-vit14.tex
\begin{table}[t]
\caption{$\ours$ consistently improves group robustness for across different efficient fine-tuning paradigms using  CLIP ViT-L/14.}\label{tab:other-ft-vit14}
\begin{adjustbox}{width=1.\linewidth}
\begin{tabular}{lcccccc}
\toprule
\multirow{2}{*}{Method} & \multicolumn{2}{c}{Waterbirds} & \multicolumn{2}{c}{CelebA} &  \multicolumn{2}{c}{MetaShift} \\
    \cmidrule(r){2-3} \cmidrule(r){4-5} \cmidrule(r){6-7} 
                        & WGA            & Avg           & WGA          & Avg         & WGA                   & Avg           \\ \midrule
Linear Probe + ERM            
&  65.9
&  97.6
&  28.3
&  94.7 
&  84.6
&  96.7            \\
\rowcolor{tabhighlight}
Linear Probe + PPA             
&  87.2             
&  94.6
&  90.4 
&  91.0
&  94.8
&  96.8       \\ \midrule   
CoOp + ERM     
& 74.0
& 97.3
& 26.7
& 94.6
& 91.9
& 96.9            \\
\rowcolor{tabhighlight}
CoOp + PPA       
& 87.4
& 94.1
& 85.6
& 88.3
& 93.7
& 96.4
\\ \midrule
Adapter + ERM     
& 79.3
& 97.8
& 54.4
& 94.5
& 90.6
& 95.5
\\
\rowcolor{tabhighlight}
Adapter + PPA               
& 83.3            
& 95.8          
& 88.3          
& 91.7          
& 92.3             
& 96.4  
\\ \bottomrule
\end{tabular}
\end{adjustbox}
\end{table}

%% file: appendix/4_additional_datasets.tex
\subsection{Additional Dataset Details}
In this section, we show the statistics of all datasets used in our experiments in~\cref{tab:waterbirds_statistics,tab:celebA_statistics,tab:metashift_statistics,tab:living17_statistics,tab:bar_statistics} and illustrate some image samples in~\cref{fig:waterbirds_exp,fig:celeba_exp,fig:metashift_images,fig:bar}.

\input{tables/dataset_details}
\input{figures/dataset_samples}
\input{figures/bar_fig}

\section{Discussion of Limitation}
Our approach assumes that the CLIP text encoder can offer class proxies for downstream tasks. However, if the pre-trained knowledge diverges significantly from the downstream tasks, the effectiveness of our method may be limited. For instance, if the images are X-ray scans and the target is to predict a specific illness, the text encoder of the pre-trained model may lack relevant medical knowledge.

%% file: tables/dataset_details.tex
\begin{table*}[t]
    \centering
    \begin{minipage}[t]{0.3\textwidth}
            \caption{Statistics of Waterbirds.}\label{tab:waterbirds_statistics}
        \centering
        \begin{adjustbox}{width=1.\linewidth}
        \begin{tabular}{lllll}
        \toprule
          & \multicolumn{2}{c}{Train} & \multicolumn{2}{c}{Test} \\ \cline{2-5} 
          & Water        & Land       & Water       & Land       \\ \midrule
Waterbird & 1057         & 56         & 642         & 642        \\
Landbird  & 184          & 3498       & 2255        & 2255      \\ \bottomrule    
\end{tabular}
\end{adjustbox}
    \end{minipage}
    \hfill
    \begin{minipage}[t]{0.32\textwidth}
            \caption{Statistics of CelebA.}\label{tab:celebA_statistics}
        \centering
        \begin{adjustbox}{width=1.\linewidth}
        \begin{tabular}{lllll}
        \toprule
          & \multicolumn{2}{c}{Train} & \multicolumn{2}{c}{Test} \\ \cline{2-5} 
          & Female        & Male       & Female       & Male       \\ \midrule
Blond      & 22880        & 1387       & 2480        & 180         \\
Non-blond  & 71629        & 66874      & 9767        & 9767     \\ \bottomrule
\end{tabular}
\end{adjustbox}
    \end{minipage}
    \hfill
    \begin{minipage}[t]{0.3\textwidth}
            \caption{Statistics of MetaShift.}\label{tab:metashift_statistics}
        \centering
        \begin{adjustbox}{width=1.\linewidth}
        \begin{tabular}{lllll}
        \toprule
          & \multicolumn{2}{c}{Train} & \multicolumn{2}{c}{Test} \\ \cline{2-5} 
          & Indoor  & Outdoor   & Indoor     & Outdoor     \\ \midrule
Cat &  630        &  153       & 345         &  65      \\
Dog  &  402       & 635       &  191        & 273     \\ \bottomrule   
\end{tabular}
\end{adjustbox}
    \end{minipage}
\end{table*}

\begin{table*}[t]
    \centering
    \begin{minipage}[t]{0.4\textwidth}
            \caption{Statistics of Living-17.}\label{tab:living17_statistics}
        \centering
        \begin{adjustbox}{width=1.\linewidth}
        \begin{tabular}{lllll} \toprule
           & \multicolumn{2}{c}{Train} & \multicolumn{2}{c}{Test} \\ \cline{2-5} 
           & Majority    & Minority    & Majority    & Minority   \\ \midrule
Group size & 2340        & 117         & 100         & 100   \\ \bottomrule
\end{tabular}
\end{adjustbox}
    \end{minipage}
    \hfill
    \begin{minipage}[t]{0.35\textwidth}
            \caption{Statistics of BAR.}\label{tab:bar_statistics}
        \centering
        \begin{adjustbox}{width=1.\linewidth}
        \begin{tabular}{lccc} \toprule
         & \multicolumn{2}{c}{Train} & Test     \\ \cline{2-4} 
         & Majority    & Minority    & Minority \\ \midrule
Climbing & 326         & 5           & 100      \\
Diving   & 520         & 8           & 151     \\
Fishing  & 163         & 4           & 38      \\
Racing   & 336         & 9           & 123        \\
Throwing & 137         & 3           &  82       \\
Vaulting & 279         & 7           &  124        \\ \bottomrule
\end{tabular}
\end{adjustbox}
    \end{minipage}
\end{table*}

%% file: figures/dataset_samples.tex
\begin{figure*}[h]
    \centering
    \begin{minipage}[t]{0.3\textwidth}
        \centering
        \includegraphics[width=\textwidth]{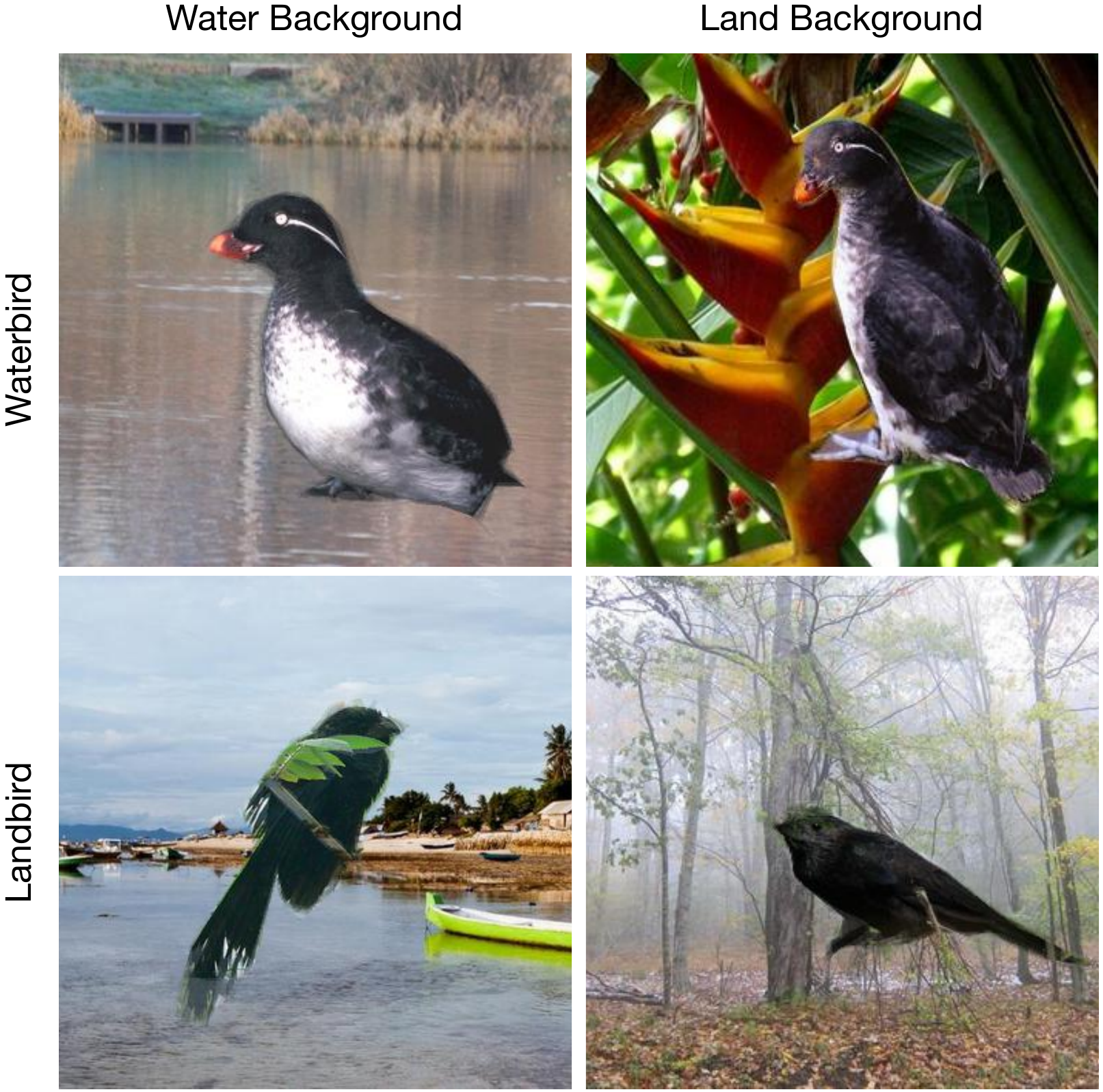}
        \caption{Image samples of Waterbirds.}    \label{fig:waterbirds_exp}
    \end{minipage}
    \hfill
    \begin{minipage}[t]{0.3\textwidth}
        \centering
        \includegraphics[width=\textwidth]{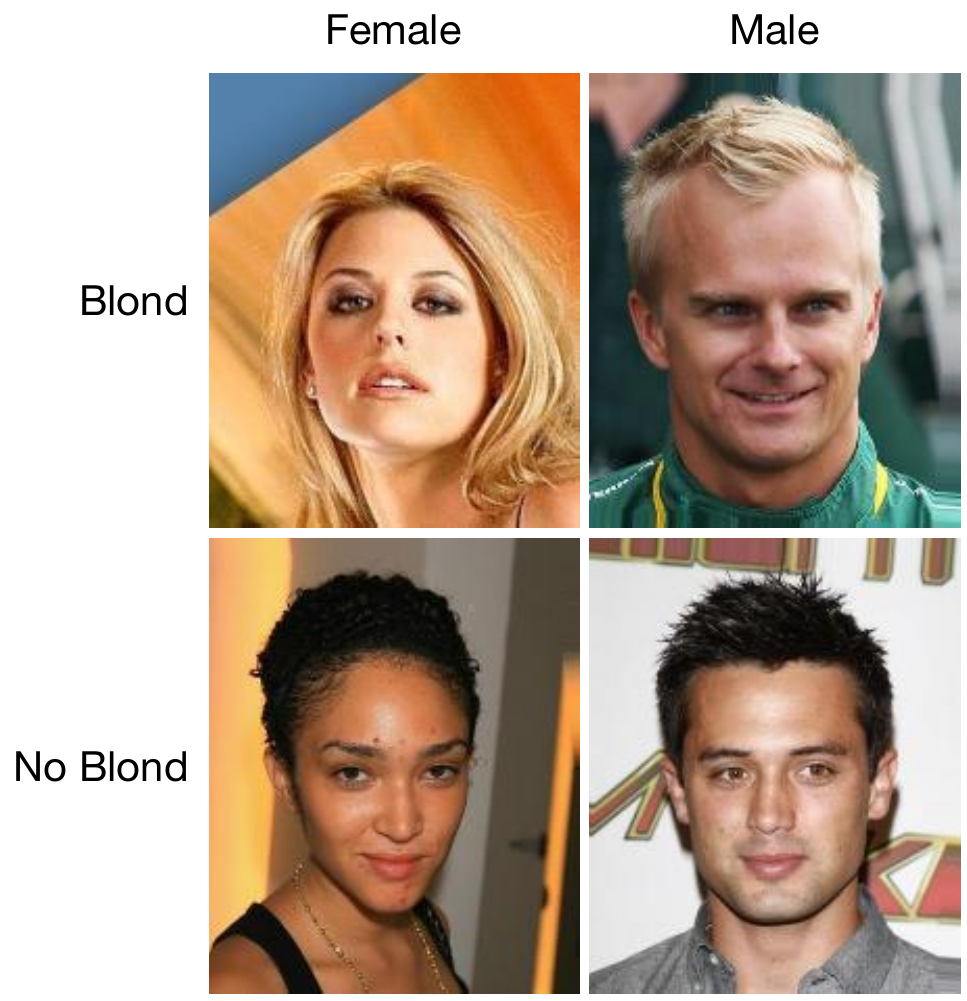}
        \caption{Image samples of CelebA.}
        \label{fig:celeba_exp}
    \end{minipage}
    \hfill
    \begin{minipage}[t]{0.3\textwidth}
        \centering
        \includegraphics[width=\textwidth]{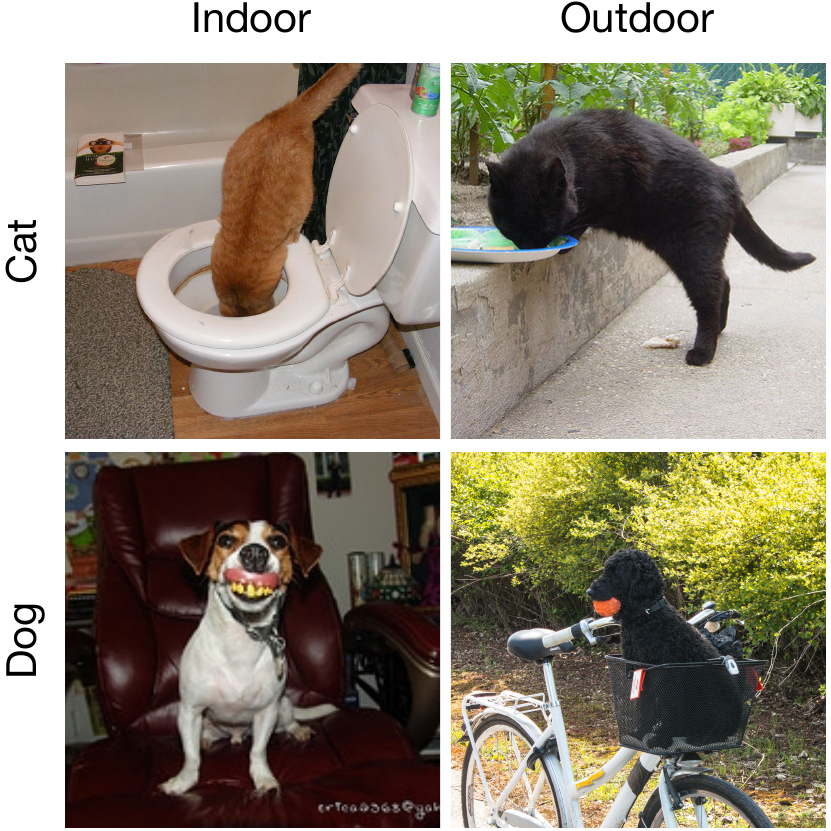}
        \caption{Image samples of MetaShift.}
        \label{fig:metashift_images}
    \end{minipage}
\end{figure*}

%% file: figures/bar_fig.tex
\begin{figure*}[t]
    \centering
    \includegraphics[width=1.0\textwidth]{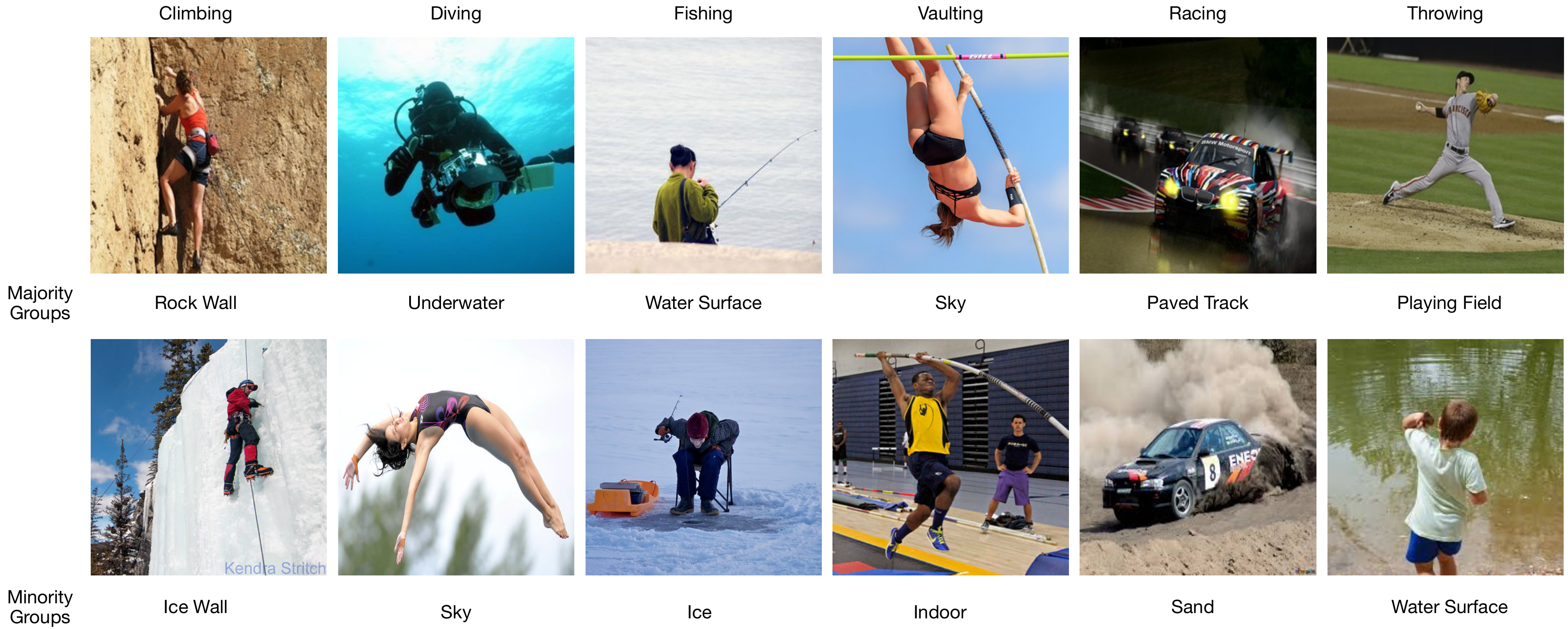}
    \caption{Image samples of BAR dataset.}
    \label{fig:bar}
\end{figure*}

%% file: main.bbl
\begin{thebibliography}{55}
\providecommand{\natexlab}[1]{#1}
\providecommand{\url}[1]{\texttt{#1}}
\expandafter\ifx\csname urlstyle\endcsname\relax
  \providecommand{\doi}[1]{doi: #1}\else
  \providecommand{\doi}{doi: \begingroup \urlstyle{rm}\Url}\fi

\bibitem[Arjovsky et~al.(2019)Arjovsky, Bottou, Gulrajani, and Lopez-Paz]{arjovsky2019invariant}
Martin Arjovsky, L{\'e}on Bottou, Ishaan Gulrajani, and David Lopez-Paz.
\newblock Invariant risk minimization.
\newblock \emph{arXiv preprint arXiv:1907.02893}, 2019.

\bibitem[Asgari et~al.(2022)Asgari, Khani, Khani, Gholami, Tran, Mahdavi~Amiri, and Hamarneh]{asgari2022masktune}
Saeid Asgari, Aliasghar Khani, Fereshte Khani, Ali Gholami, Linh Tran, Ali Mahdavi~Amiri, and Ghassan Hamarneh.
\newblock Masktune: Mitigating spurious correlations by forcing to explore.
\newblock In \emph{NeurIPS}, 2022.

\bibitem[Chuang et~al.(2023)Chuang, Jampani, Li, Torralba, and Jegelka]{chuang2023debiasing}
Ching-Yao Chuang, Varun Jampani, Yuanzhen Li, Antonio Torralba, and Stefanie Jegelka.
\newblock Debiasing vision-language models via biased prompts.
\newblock \emph{arXiv preprint arXiv:2302.00070}, 2023.

\bibitem[Creager et~al.(2021)Creager, Jacobsen, and Zemel]{creager2021environment}
Elliot Creager, J{\"o}rn-Henrik Jacobsen, and Richard Zemel.
\newblock Environment inference for invariant learning.
\newblock In \emph{ICML}, 2021.

\bibitem[Dagaev et~al.(2023)Dagaev, Roads, Luo, Barry, Patil, and Love]{dagaev2023too}
Nikolay Dagaev, Brett~D Roads, Xiaoliang Luo, Daniel~N Barry, Kaustubh~R Patil, and Bradley~C Love.
\newblock A too-good-to-be-true prior to reduce shortcut reliance.
\newblock \emph{Pattern recognition letters}, 166:\penalty0 164--171, 2023.

\bibitem[Espinosa~Zarlenga et~al.(2024)Espinosa~Zarlenga, Sankaranarayanan, Andrews, Shams, Jamnik, and Xiang]{espinosa2024efficient}
Mateo Espinosa~Zarlenga, Swami Sankaranarayanan, Jerone~TA Andrews, Zohreh Shams, Mateja Jamnik, and Alice Xiang.
\newblock Efficient bias mitigation without privileged information.
\newblock In \emph{ECCV}, 2024.

\bibitem[Gao et~al.(2024)Gao, Geng, Zhang, Ma, Fang, Zhang, Li, and Qiao]{gao2024clip}
Peng Gao, Shijie Geng, Renrui Zhang, Teli Ma, Rongyao Fang, Yongfeng Zhang, Hongsheng Li, and Yu Qiao.
\newblock Clip-adapter: Better vision-language models with feature adapters.
\newblock \emph{IJCV}, 2024.

\bibitem[Geirhos et~al.(2020)Geirhos, Jacobsen, Michaelis, Zemel, Brendel, Bethge, and Wichmann]{geirhos2020shortcut}
Robert Geirhos, J{\"o}rn-Henrik Jacobsen, Claudio Michaelis, Richard Zemel, Wieland Brendel, Matthias Bethge, and Felix~A Wichmann.
\newblock Shortcut learning in deep neural networks.
\newblock \emph{Nature Machine Intelligence}, 2\penalty0 (11):\penalty0 665--673, 2020.

\bibitem[Houlsby et~al.(2019)Houlsby, Giurgiu, Jastrzebski, Morrone, De~Laroussilhe, Gesmundo, Attariyan, and Gelly]{houlsby2019parameter}
Neil Houlsby, Andrei Giurgiu, Stanislaw Jastrzebski, Bruna Morrone, Quentin De~Laroussilhe, Andrea Gesmundo, Mona Attariyan, and Sylvain Gelly.
\newblock Parameter-efficient transfer learning for nlp.
\newblock In \emph{ICML}, 2019.

\bibitem[Idrissi et~al.(2022)Idrissi, Arjovsky, Pezeshki, and Lopez-Paz]{idrissi2022simple}
Badr~Youbi Idrissi, Martin Arjovsky, Mohammad Pezeshki, and David Lopez-Paz.
\newblock Simple data balancing achieves competitive worst-group-accuracy.
\newblock In \emph{Conference on Causal Learning and Reasoning}, 2022.

\bibitem[Izmailov et~al.(2022)Izmailov, Kirichenko, Gruver, and Wilson]{izmailov2022feature}
Pavel Izmailov, Polina Kirichenko, Nate Gruver, and Andrew~G Wilson.
\newblock On feature learning in the presence of spurious correlations.
\newblock In \emph{NeurIPS}, 2022.

\bibitem[Jia et~al.(2021)Jia, Yang, Xia, Chen, Parekh, Pham, Le, Sung, Li, and Duerig]{jia2021scaling}
Chao Jia, Yinfei Yang, Ye Xia, Yi-Ting Chen, Zarana Parekh, Hieu Pham, Quoc Le, Yun-Hsuan Sung, Zhen Li, and Tom Duerig.
\newblock Scaling up visual and vision-language representation learning with noisy text supervision.
\newblock In \emph{ICML}, 2021.

\bibitem[Jung et~al.(2022)Jung, Chun, and Moon]{jung2022learning}
Sangwon Jung, Sanghyuk Chun, and Taesup Moon.
\newblock Learning fair classifiers with partially annotated group labels.
\newblock In \emph{CVPR}, 2022.

\bibitem[Kim et~al.(2021)Kim, Lee, and Choo]{kim2021biaswap}
Eungyeup Kim, Jihyeon Lee, and Jaegul Choo.
\newblock Biaswap: Removing dataset bias with bias-tailored swapping augmentation.
\newblock In \emph{ICCV}, 2021.

\bibitem[Kim et~al.(2022)Kim, Hwang, Ahn, Park, and Kwak]{kim2022learning}
Nayeong Kim, Sehyun Hwang, Sungsoo Ahn, Jaesik Park, and Suha Kwak.
\newblock Learning debiased classifier with biased committee.
\newblock In \emph{NeurIPS}, 2022.

\bibitem[Kim et~al.(2024)Kim, Mo, Kim, Lee, Lee, and Shin]{kim2024discovering}
Younghyun Kim, Sangwoo Mo, Minkyu Kim, Kyungmin Lee, Jaeho Lee, and Jinwoo Shin.
\newblock Discovering and mitigating visual biases through keyword explanation.
\newblock In \emph{CVPR}, 2024.

\bibitem[Kirichenko et~al.(2022)Kirichenko, Izmailov, and Wilson]{kirichenko2022last}
Polina Kirichenko, Pavel Izmailov, and Andrew~Gordon Wilson.
\newblock Last layer re-training is sufficient for robustness to spurious correlations.
\newblock In \emph{ICML}, 2022.

\bibitem[Lee et~al.(2021)Lee, Kim, Lee, Lee, and Choo]{lee2021learning}
Jungsoo Lee, Eungyeup Kim, Juyoung Lee, Jihyeon Lee, and Jaegul Choo.
\newblock Learning debiased representation via disentangled feature augmentation.
\newblock In \emph{NeurIPS}, 2021.

\bibitem[Li and Vasconcelos(2019)]{li2019repair}
Yi Li and Nuno Vasconcelos.
\newblock Repair: Removing representation bias by dataset resampling.
\newblock In \emph{CVPR}, 2019.

\bibitem[Liang and Zou(2022)]{liangmetashift}
Weixin Liang and James Zou.
\newblock Metashift: A dataset of datasets for evaluating contextual distribution shifts and training conflicts.
\newblock In \emph{ICLR}, 2022.

\bibitem[Liu et~al.(2021)Liu, Haghgoo, Chen, Raghunathan, Koh, Sagawa, Liang, and Finn]{liu2021just}
Evan~Z Liu, Behzad Haghgoo, Annie~S Chen, Aditi Raghunathan, Pang~Wei Koh, Shiori Sagawa, Percy Liang, and Chelsea Finn.
\newblock Just train twice: Improving group robustness without training group information.
\newblock In \emph{ICML}, 2021.

\bibitem[Liu et~al.(2023)Liu, Zhang, Sekhar, Wu, Singhal, and Fernandez-Granda]{liuavoiding}
Sheng Liu, Xu Zhang, Nitesh Sekhar, Yue Wu, Prateek Singhal, and Carlos Fernandez-Granda.
\newblock Avoiding spurious correlations via logit correction.
\newblock In \emph{ICLR}, 2023.

\bibitem[Liu et~al.(2015)Liu, Luo, Wang, and Tang]{liu2015deep}
Ziwei Liu, Ping Luo, Xiaogang Wang, and Xiaoou Tang.
\newblock Deep learning face attributes in the wild.
\newblock In \emph{ICCV}, 2015.

\bibitem[Menon et~al.(2020)Menon, Rawat, and Kumar]{menon2020overparameterisation}
Aditya~Krishna Menon, Ankit~Singh Rawat, and Sanjiv Kumar.
\newblock Overparameterisation and worst-case generalisation: friend or foe?
\newblock In \emph{ICLR}, 2020.

\bibitem[Menon et~al.(2021)Menon, Jayasumana, Rawat, Jain, Veit, and Kumar]{menonlong}
Aditya~Krishna Menon, Sadeep Jayasumana, Ankit~Singh Rawat, Himanshu Jain, Andreas Veit, and Sanjiv Kumar.
\newblock Long-tail learning via logit adjustment.
\newblock In \emph{ICLR}, 2021.

\bibitem[Nam et~al.(2020)Nam, Cha, Ahn, Lee, and Shin]{nam2020learning}
Junhyun Nam, Hyuntak Cha, Sungsoo Ahn, Jaeho Lee, and Jinwoo Shin.
\newblock Learning from failure: De-biasing classifier from biased classifier.
\newblock In \emph{NeurIPS}, 2020.

\bibitem[Nam et~al.(2022)Nam, Kim, Lee, and Shin]{namspread}
Junhyun Nam, Jaehyung Kim, Jaeho Lee, and Jinwoo Shin.
\newblock Spread spurious attribute: Improving worst-group accuracy with spurious attribute estimation.
\newblock In \emph{ICLR}, 2022.

\bibitem[Phan et~al.(2024{\natexlab{a}})Phan, Wilson, and Lei]{phan2024controllable}
Hoang Phan, Andrew~Gordon Wilson, and Qi Lei.
\newblock Controllable prompt tuning for balancing group distributional robustness.
\newblock \emph{arXiv preprint arXiv:2403.02695}, 2024{\natexlab{a}}.

\bibitem[Phan et~al.(2024{\natexlab{b}})Phan, Wilson, and Lei]{phancontrollable}
Hoang Phan, Andrew~Gordon Wilson, and Qi Lei.
\newblock Controllable prompt tuning for balancing group distributional robustness.
\newblock In \emph{ICML}, 2024{\natexlab{b}}.

\bibitem[Qiu et~al.(2023)Qiu, Potapczynski, Izmailov, and Wilson]{qiu2023simple}
Shikai Qiu, Andres Potapczynski, Pavel Izmailov, and Andrew~Gordon Wilson.
\newblock Simple and fast group robustness by automatic feature reweighting.
\newblock In \emph{International Conference on Machine Learning}, pages 28448--28467. PMLR, 2023.

\bibitem[Radford et~al.(2021)Radford, Kim, Hallacy, Ramesh, Goh, Agarwal, Sastry, Askell, Mishkin, Clark, et~al.]{radford2021learning}
Alec Radford, Jong~Wook Kim, Chris Hallacy, Aditya Ramesh, Gabriel Goh, Sandhini Agarwal, Girish Sastry, Amanda Askell, Pamela Mishkin, Jack Clark, et~al.
\newblock Learning transferable visual models from natural language supervision.
\newblock In \emph{ICML}, 2021.

\bibitem[Sagawa et~al.(2019)Sagawa, Koh, Hashimoto, and Liang]{sagawa2019distributionally}
Shiori Sagawa, Pang~Wei Koh, Tatsunori~B Hashimoto, and Percy Liang.
\newblock Distributionally robust neural networks.
\newblock In \emph{ICLR}, 2019.

\bibitem[Santurkar et~al.(2021)Santurkar, Tsipras, and Madry]{santurkar2021breeds}
Shibani Santurkar, Dimitris Tsipras, and Aleksander Madry.
\newblock Breeds: Benchmarks for subpopulation shift.
\newblock In \emph{ICLR}, 2021.

\bibitem[Taghanaki et~al.(2021)Taghanaki, Choi, Khasahmadi, and Goyal]{taghanaki2021robust}
Saeid~A Taghanaki, Kristy Choi, Amir~Hosein Khasahmadi, and Anirudh Goyal.
\newblock Robust representation learning via perceptual similarity metrics.
\newblock In \emph{ICML}, 2021.

\bibitem[Tartaglione et~al.(2021)Tartaglione, Barbano, and Grangetto]{tartaglione2021end}
Enzo Tartaglione, Carlo~Alberto Barbano, and Marco Grangetto.
\newblock End: Entangling and disentangling deep representations for bias correction.
\newblock In \emph{CVPR}, 2021.

\bibitem[Tsirigotis et~al.(2024)Tsirigotis, Monteiro, Rodriguez, Vazquez, and Courville]{tsirigotis2024group}
Christos Tsirigotis, Joao Monteiro, Pau Rodriguez, David Vazquez, and Aaron~C Courville.
\newblock Group robust classification without any group information.
\newblock \emph{NeurIPS}, 2024.

\bibitem[Vapnik(1991)]{vapnik1991principles}
Vladimir Vapnik.
\newblock Principles of risk minimization for learning theory.
\newblock In \emph{NeurIPS}, 1991.

\bibitem[Wah et~al.(2011)Wah, Branson, Welinder, Perona, and Belongie]{wah2011caltech}
Catherine Wah, Steve Branson, Peter Welinder, Pietro Perona, and Serge Belongie.
\newblock The caltech-ucsd birds-200-2011 dataset.
\newblock 2011.

\bibitem[Wang et~al.(2020)Wang, Qinami, Karakozis, Genova, Nair, Hata, and Russakovsky]{wang2020towards}
Zeyu Wang, Klint Qinami, Ioannis~Christos Karakozis, Kyle Genova, Prem Nair, Kenji Hata, and Olga Russakovsky.
\newblock Towards fairness in visual recognition: Effective strategies for bias mitigation.
\newblock In \emph{CVPR}, 2020.

\bibitem[Wang et~al.(2022)Wang, Yu, Yu, Dai, Tsvetkov, and Cao]{wang2022simvlm}
Zirui Wang, Jiahui Yu, Adams~Wei Yu, Zihang Dai, Yulia Tsvetkov, and Yuan Cao.
\newblock Simvlm: Simple visual language model pretraining with weak supervision.
\newblock In \emph{ICLR}, 2022.

\bibitem[Wortsman et~al.(2022)Wortsman, Ilharco, Kim, Li, Kornblith, Roelofs, Lopes, Hajishirzi, Farhadi, Namkoong, et~al.]{wortsman2022robust}
Mitchell Wortsman, Gabriel Ilharco, Jong~Wook Kim, Mike Li, Simon Kornblith, Rebecca Roelofs, Raphael~Gontijo Lopes, Hannaneh Hajishirzi, Ali Farhadi, Hongseok Namkoong, et~al.
\newblock Robust fine-tuning of zero-shot models.
\newblock In \emph{CVPR}, 2022.

\bibitem[Wu et~al.(2023)Wu, Yuksekgonul, Zhang, and Zou]{wu2023discover}
Shirley Wu, Mert Yuksekgonul, Linjun Zhang, and James Zou.
\newblock Discover and cure: Concept-aware mitigation of spurious correlation.
\newblock In \emph{ICML}, 2023.

\bibitem[Yaghoobzadeh et~al.(2019)Yaghoobzadeh, Mehri, Tachet, Hazen, and Sordoni]{yaghoobzadeh2019increasing}
Yadollah Yaghoobzadeh, Soroush Mehri, Remi Tachet, Timothy~J Hazen, and Alessandro Sordoni.
\newblock Increasing robustness to spurious correlations using forgettable examples.
\newblock \emph{arXiv preprint arXiv:1911.03861}, 2019.

\bibitem[Yang et~al.(2023{\natexlab{a}})Yang, Nushi, Palangi, and Mirzasoleiman]{yang2023mitigating}
Yu Yang, Besmira Nushi, Hamid Palangi, and Baharan Mirzasoleiman.
\newblock Mitigating spurious correlations in multi-modal models during fine-tuning.
\newblock In \emph{ICML}, 2023{\natexlab{a}}.

\bibitem[Yang et~al.(2023{\natexlab{b}})Yang, Zhang, Katabi, and Ghassemi]{yang2023change}
Yuzhe Yang, Haoran Zhang, Dina Katabi, and Marzyeh Ghassemi.
\newblock Change is hard: A closer look at subpopulation shift.
\newblock \emph{arXiv preprint arXiv:2302.12254}, 2023{\natexlab{b}}.

\bibitem[Yang et~al.(2022)Yang, Chou, and Chaudhuri]{yang2022understanding}
Yao-Yuan Yang, Chi-Ning Chou, and Kamalika Chaudhuri.
\newblock Understanding rare spurious correlations in neural networks.
\newblock \emph{arXiv preprint arXiv:2202.05189}, 2022.

\bibitem[You et~al.(2024)You, Min, Dai, Sekhon, Staib, and Duncan]{you2024calibrating}
Chenyu You, Yifei Min, Weicheng Dai, Jasjeet~S Sekhon, Lawrence Staib, and James~S Duncan.
\newblock Calibrating multi-modal representations: A pursuit of group robustness without annotations.
\newblock In \emph{CVPR}, 2024.

\bibitem[Zhang and R{\'e}(2022)]{zhang2022contrastive}
Michael Zhang and Christopher R{\'e}.
\newblock Contrastive adapters for foundation model group robustness.
\newblock In \emph{NeurIPS}, 2022.

\bibitem[Zhang et~al.(2022)Zhang, Sohoni, Zhang, Finn, and Re]{zhang2022correct}
Michael Zhang, Nimit~S Sohoni, Hongyang~R Zhang, Chelsea Finn, and Christopher Re.
\newblock Correct-n-contrast: a contrastive approach for improving robustness to spurious correlations.
\newblock In \emph{ICML}, 2022.

\bibitem[Zhang et~al.(2023)Zhang, HaoChen, Huang, Wang, Zou, and Yeung]{zhang2023diagnosing}
Yuhui Zhang, Jeff~Z HaoChen, Shih-Cheng Huang, Kuan-Chieh Wang, James Zou, and Serena Yeung.
\newblock Diagnosing and rectifying vision models using language.
\newblock \emph{arXiv preprint arXiv:2302.04269}, 2023.

\bibitem[Zhou et~al.(2017)Zhou, Lapedriza, Khosla, Oliva, and Torralba]{zhou2017places}
Bolei Zhou, Agata Lapedriza, Aditya Khosla, Aude Oliva, and Antonio Torralba.
\newblock Places: A 10 million image database for scene recognition.
\newblock \emph{IEEE TPAMI}, 40\penalty0 (6):\penalty0 1452--1464, 2017.

\bibitem[Zhou et~al.(2022)Zhou, Yang, Loy, and Liu]{zhou2022learning}
Kaiyang Zhou, Jingkang Yang, Chen~Change Loy, and Ziwei Liu.
\newblock Learning to prompt for vision-language models.
\newblock \emph{IJCV}, 2022.

\bibitem[Zhu et~al.(2023{\natexlab{a}})Zhu, Niu, Han, Wu, and Zhang]{zhu2023prompt}
Beier Zhu, Yulei Niu, Yucheng Han, Yue Wu, and Hanwang Zhang.
\newblock Prompt-aligned gradient for prompt tuning.
\newblock In \emph{ICCV}, 2023{\natexlab{a}}.

\bibitem[Zhu et~al.(2023{\natexlab{b}})Zhu, Tang, Sun, and Zhang]{zhu2023generalized}
Beier Zhu, Kaihua Tang, Qianru Sun, and Hanwang Zhang.
\newblock Generalized logit adjustment: Calibrating fine-tuned models by removing label bias in foundation models.
\newblock In \emph{NeurIPS}, 2023{\natexlab{b}}.

\bibitem[Zhu et~al.(2024)Zhu, Cui, and Zhang]{zhu2024robust}
Beier Zhu, Jiequan Cui, and Hanwang Zhang.
\newblock Robust fine-tuning of zero-shot models via variance reduction.
\newblock In \emph{NeurIPS}, 2024.

\end{thebibliography}
